\documentclass[letterpaper, 10 pt, conference]{ieeeconf}  % Comment this line out if you need a4paper

\IEEEoverridecommandlockouts                              % This command is only needed if 
\usepackage{amsmath}
\usepackage{amssymb}
\usepackage{graphicx}
\usepackage[sort,compress]{cite}
\usepackage{booktabs}

\usepackage{amsthm}

\newtheorem{assumption}{Assumption}
\newtheorem{theorem}{Theorem}
\newtheorem{problem}{Problem}
\newtheorem{lemma}{Lemma}

\newtheorem{remark}{Remark}
\newtheorem{example}{Example}

\usepackage[dvipsnames]{xcolor}
\usepackage{subcaption}
\usepackage{wrapfig}
\usepackage{multirow}
\usepackage{mathtools}
\usepackage{empheq}
\usepackage{color}
\usepackage{bbm}
\usepackage{xspace}
\let\labelindent\relax
\usepackage[shortlabels]{enumitem}

\usepackage{algorithmicx}
\usepackage{algorithm}
\usepackage[]{algpseudocode}

\usepackage{fontawesome}

\usepackage{soul}
\usepackage{comment}

\usepackage[hyphens]{url}
\usepackage[pdftex, colorlinks=true, linkcolor=blue, citecolor = black, urlcolor = black, filecolor=black , pagebackref=false, hypertexnames=false]{hyperref}
\usepackage{cleveref}
\crefname{example}{Example}{Examples}
\crefname{section}{Sec.}{Secs.}
\Crefname{section}{Section}{Sections}
\Crefname{table}{Table}{Tables}
\crefname{table}{Table}{Tabs.}
\crefname{figure}{Fig.}{Figs.}
\crefname{algorithm}{Algorithm}{Algorithms}
\crefname{remark}{Remark}{Remarks}
\crefname{theorem}{Theorem}{Theorems}
\crefname{lemma}{Lemma}{Lemmas}

\usepackage[font=small]{caption}
\newcommand{\Real}{\mathbb{R}}
\newcommand{\Natural}{\mathbb{N}}

\DeclareMathOperator*{\argmin}{arg\,min}

\providecommand{\inner}[2]{\left \langle #1, #2 \right \rangle}
\providecommand{\norm}[1]{\left\|#1\right\|}
\providecommand{\twonorm}[1]{\left\|#1\right\|_2}

\DeclareMathOperator{\Diag}{Diag}
\DeclareMathOperator{\rgrad}{grad}
\DeclareMathOperator{\Retr}{Retr}
\providecommand{\transporter}[2]{T^{#2\leftarrow#1}}

\DeclareMathOperator{\SE}{SE}

\newcommand{\Algorithm}{{LARPG}\xspace}

\providecommand{\dataset}[1]{{\small \textsf{#1}}\xspace}
\providecommand{\software}[1]{{#1}\xspace}

\newcommand{\etal}{\emph{et~al.}\xspace}

\newcommand{\eg}{\emph{e.g.}\xspace}
\newcommand{\ie}{\emph{i.e.}\xspace}

\newcommand{\cf}{\emph{c.f.}\xspace}

\newcommand{\Ycal}{\mathcal{Y}}
\newcommand{\Xcal}{\mathcal{X}}

\newcommand{\bmat}{\begin{bmatrix}}
\newcommand{\emat}{\end{bmatrix}}	
\newcommand{\fhat}{\widehat{f}}
\newcommand{\hhat}{\widehat{h}}
\newcommand{\mhat}{\widehat{m}}
\newcommand{\what}{\widehat{w}}
\newcommand{\wtilde}{\widetilde{w}}
\newcommand{\ustar}{u^\star}
\newcommand{\vstar}{v^\star}
\newcommand{\sigmabar}{\sigma_p}

\newcommand{\myParagraph}[1]{{\bf #1.}\xspace}

\providecommand{\optional}[1]{{#1}}
\providecommand{\remove}[1]{{}}
\providecommand{\revise}[1]{{#1}}

\title{\LARGE \bf
Distributed Riemannian Optimization with Lazy Communication\\% 
for Collaborative Geometric Estimation
}

\author{Yulun Tian$^{1}$, Amrit Singh Bedi$^{2}$, Alec Koppel$^{3}$, Miguel Calvo-Fullana$^{1}$, \\ David M. Rosen$^{4}$, and Jonathan P. How$^{1}$% <-this % stops a space
\thanks{*This work was supported in part by ARL DCIST under Cooperative Agreement Number W911NF-17-2-0181, and in part by ONR under BRC award N000141712072.}% <-this % stops a space
\thanks{$^{1}$Y. Tian, M. Calvo-Fullana, and J. P. How are with the Department of Aeronautics and Astronautics, Massachusetts Institute of Technology, 77 Massachusetts Ave, Cambridge, MA 02139, USA {\tt\small \{yulun, cfullana, jhow\}@mit.edu}}%
\thanks{$^{2}${A. S. Bedi is with Institute of Systems Research, \revise{University Of Maryland}, College Park, MD, USA.}
        {\tt\small amritbd@umd.edu}}%
\thanks{$^{3}${A. Koppel is with Supply Chain Optimization Technologies,
Amazon. 320 108th Avenue NE, Bellevue, WA 98004, USA.}
    {\tt\small aekoppel@amazon.com}}%
\thanks{$^{4}${D. M. Rosen is with the Departments of Electrical and Computer Engineering and Mathematics, Northeastern University, 360 Huntington Ave, Boston, MA 02115, USA.}
	{\tt\small d.rosen@northeastern.edu}}%
}

\begin{document}

\maketitle
%\thispagestyle{empty}
%\pagestyle{empty}

%%%%%%%%%%%%%%%%%%%%%%%%%%%%%%%%%%%%%%%%%%%%%%%%%%%%%%%%%%%%%%%%%%%%%%%%%%%%%%%%
\begin{abstract}
	We present the first distributed optimization algorithm with \emph{lazy communication} for collaborative geometric estimation, the backbone of modern collaborative simultaneous localization and mapping (SLAM) and structure-from-motion (SfM) applications. 
	Our method allows agents to cooperatively reconstruct a shared geometric model on a central server by fusing individual observations, 
	but \emph{without} the need to transmit potentially sensitive information about the agents themselves (such as their locations).
	Furthermore, to alleviate the burden of communication during iterative optimization, we design a set of \emph{communication triggering conditions} that enable agents to selectively upload \revise{a targeted subset of} local information that \revise{is} useful to global optimization. 
	Our approach thus achieves significant communication reduction with minimal impact on optimization performance.
	As our main theoretical contribution, we prove that our method converges to first-order critical points with a \revise{global} sublinear convergence rate.
	Numerical evaluations on bundle adjustment problems from collaborative SLAM and SfM datasets show that our method performs competitively against existing distributed techniques, while achieving up to 78\% total communication reduction.
\end{abstract}

% Force the use of "et al." for references with too many authors
\bstctlcite{bst_control}

%!TEX root = ../root.tex
\section{Introduction}
Geometric estimation, which refers to the task of estimating geometric models (\eg, poses and 3D structure) from multiple views, is a fundamental technology that underlies important robotic applications such as simultaneous localization and mapping (SLAM) and Structure-from-Motion (SfM). 
For emerging applications in multi-robot systems 
and mixed reality,
\emph{collaborative} geometric estimation enables multiple agents to build and use a shared geometric model (\eg, a large-scale 3D map). 
At the core of this process is a large-scale optimization that fuses measurements collected by all agents to produce a global geometric model. 

Existing multi-agent systems often offload the aforementioned global optimization, such as bundle adjustment (BA), to a central server or base station \cite{Forster13CSLAM,Schmuck18CCM,Ebadi20LAMP,Schmuck21covins}. 
However, as the number of robots or mission time increases, 
centralized optimization suffers from increasing problem size that eventually makes the server a computational bottleneck.
Furthermore, centralized optimization usually requires agents to communicate \emph{private} data (\eg, images or locations) to the server, which compromises privacy requirements in applications such as autonomous driving and mixed reality.

\emph{Distributed optimization} provides a promising solution that addresses both scalability and privacy concerns by leveraging the {local} computational power of the agents.
In a distributed architecture, agents collaboratively solve the underlying optimization problem by coordinating with the server or with each other directly.
However, distributed systems often require more frequent communication than their centralized counterparts due to the iterative nature of most optimization algorithms.
Furthermore, the amount of data communicated at each iteration often grows proportionally with the dimension of the shared model.
For large models, this type of \emph{iterative communication} can result in long delays under real-world communication networks.
Consequently, existing distributed systems often use simpler formulations that require less communication (\eg, pose graph optimization  \cite{Cieslewski18DataEfficient,Lajoie20DOOR,Tian21KimeraMulti})
or operate on computer clusters with high-performance communication \cite{Zhang17DistBA}.

In this work, we develop a \emph{communication-efficient} algorithm for collaborative geometric estimation, which significantly reduces the burden of communication when performing distributed optimization on high-dimensional problems. 
The core idea behind our approach is \emph{lazy communication}: 
instead of uploading all information at every iteration, 
agents selectively upload parts of their local information that have changed significantly from the past. 
While the main idea is intuitive, incorporating lazy communication in our applications raises a series of technical questions ranging from algorithm design to theoretical analysis of convergence that we address in this work.

\myParagraph{Contributions}
We propose a \emph{communication-efficient} distributed Riemannian optimization algorithm for collaborative geometric estimation. 
To tackle the numerical poor conditioning associated with most real-world problems,
we design a distributed method that performs \emph{approximate second-order} updates 
while simultaneously protecting the privacy of participating agents.
Furthermore, we augment our basic method with \emph{lazy communication}, which enables agents to only transmit the parts of their local information that satisfy certain \revise{\emph{communication triggering conditions}}, and hence significantly reduces overall communication.
We prove that our final algorithm converges globally to first-order critical points \revise{with a global sublinear rate}.
Compared to related \revise{works} that study lazy communication in distributed first-order methods (\eg, \cite{Chen18LAG}), our algorithm design and convergence analysis 
are significantly different and account for the employed second-order updates, the treatment of non-convex manifold constraints, among other details (see Remark~\ref{remark:novelty}).
We perform extensive evaluations on large-scale BA problems in collaborative SLAM and SfM scenarios, which are central to emerging multi-robot navigation and mixed reality applications.
Results show that our algorithm achieves competitive performance compared to other state-of-the-art methods under the same communication architecture, while achieving up to 78\% total communication reduction.

\subsection*{Preliminaries on Riemannian Optimization}
For a smooth Riemannian manifold $\Xcal$, we denote the {tangent space} at $x \in \Xcal$ as $T_x \Xcal$.
For two tangent vectors $u_1, u_2 \in T_x \Xcal$, the inner product is denoted as $\inner{u_1}{u_2}_x$, and the corresponding norm is $\norm{u}_x \triangleq \sqrt{\inner{u}{u}_x}$.
In the rest of the paper, we drop the subscript $x$ as it will be clear from context.
Let $M: T_{x_1} \Xcal \to T_{x_2} \Xcal$ be a linear map between two tangent spaces.
With a slight abuse of notation, we also use $M$ to denote the matrix representation of this linear map under chosen bases of $T_{x_1} \Xcal$ and $T_{x_2} \Xcal$. 
For $u \in T_{x_1} \Xcal$, $Mu \in T_{x_2} \Xcal$ denotes the result of applying $M$ on $u$.
Further, $\norm{M}$ denotes the operator norm of $M$ with respect to the Riemannian metric. 
When $M: T_{x} \Xcal \to T_x \Xcal$ maps a tangent space to itself and is
symmetric and positive definite, we define its associated inner product
as $\inner{u_1}{u_2}_M \triangleq \inner{u_1}{Mu_2}$ with the corresponding norm 
$\norm{u}_M \triangleq \sqrt{\inner{u}{u}_M}$. 
A \emph{retraction} at $x$ is a smooth map $\Retr_x: T_x \Xcal \to \Xcal$ that preserves the first-order geometry of $\Xcal$.
For a scalar function defined on the manifold $f: \Xcal \to \Real$, we use $\rgrad f(x) \in T_x \Xcal$ to denote its Riemannian gradient at $x \in \Xcal$.
Intuitively, $\rgrad f(x)$ provides the direction of steepest ascent in the tangent space at $x$. 
The reader is referred to \cite{Absil2009Book,Boumal20Book} for a more rigorous treatment of Riemannian optimization.

%!TEX root = ../root.tex
\section{Related Work}

Centralized geometric estimation is a well studied subject with off-the-shelf high-performance solvers available \cite{gtsam,g2o,ceres-solver}.
Recently, Zhang~\etal~\cite{Zhang21MRiSAM2} develop a centralized incremental solver for multi-robot SLAM.
Meanwhile, distributed methods have gained increasing attention; 
see \cite{Lajoie21Survey} for a recent survey.
Cunningham~\etal develop the pioneering work of DDF-SAM~\cite{Cunningham10DDFSAM,Cunningham13DDFSAM2}, 
where agents use Gaussian elimination to exchange marginals over commonly observed landmarks.
\optional{Later work build on similar idea and propose extensions such as consistent sparsification \cite{Paull15Sparsification} and real-time operation on devices with limited resources \cite{Zhang21ServerClient}.}
Our proposed method employs a similar elimination technique, and furthermore supports lazy communication to achieve significant communication reduction. 
\optional{Another recent line of research investigates distributed pose graph SLAM~\cite{Tron2014CameraNetwork,Choudhary17IJRR,Tian21DC2PGO,Tian20ASAPP,Fan20Majorization}. However, in most cases these methods only estimate robots' poses and not the global map.}
Related work in computer vision considers solving large-scale SfM using distributed architectures. 
Earlier work proposes to use distributed conjugate gradients for multi-core BA \cite{Wu11MulticoreBA}.
More recently, researchers have proposed alternative algorithms based on Douglas-Rachford splitting~\cite{Eriksson16ConsensusBA} or alternating direction method of multipliers (ADMM)~\cite{Zhang17DistBA}.

Communication efficiency has been a central theme in distributed optimization.
\optional{Recently, this topic has gained increasing attention due to the success of federated learning \cite{FL21Advances}. }
Multiple techniques to achieve communication efficiency have been proposed, including the use of quantization \cite{suresh2017distributed} and distributed second-order methods \cite{shamir2014communication}. 
In this work, we explore an alternative strategy based on \emph{lazy} or \emph{event-triggered} communication, which has demonstrated impressive results \cite{Chen18LAG}.
\optional{The same idea has found successful applications in related areas such as decentralized control \cite{Ding18EventTrigger}.}
We develop lazy communication schemes for collaborative geometric estimation, 
which requires substantial innovations in algorithm design and theoretical analysis compared to existing work \cite{Chen18LAG}; see Remark~\ref{remark:novelty}.

%!TEX root = ../root.tex
\section{Collaborative Geometric Estimation}
\label{sec:problem_formulation}
We consider a scenario where $N$ agents navigate in a common environment, 
and seek to collaboratively estimate a shared geometric model $y \in \Ycal$.
For this purpose, agents communicate with a central server,
who is responsible for coordinating updates across the team. 
In practice, the server can be a lead agent or a base station.
Motivated by most real-world applications, we assume that the shared model $y$ consists of $m$ smaller elements $y = \{y_1, \hdots, y_m\}$, where each element $y_l \in \Ycal_l$ corresponds to a single geometric primitive.
For instance, when $y$ represents a point cloud map, each $y_l$ corresponds to a single 3D point. 
During navigation, each agent $i \in [N] \triangleq \{1, \hdots, N\}$ observes a subset of the shared model.
In addition, agent $i$ also maintains \emph{private} variable $x_i \in \Xcal_i$, which can contain sensitive information such as the trajectory of this agent.

In this work, we focus on solving the maximum likelihood estimation (MLE) problem in the multi-agent scenario described above. 
The MLE formulation is very general and encompasses a wide range of robot perception problems \cite{Dellaert17FactorGraph}.
Under the MLE formulation, local measurements collected by agent $i$ induce a local cost \revise{function}\footnote{For clarity of presentation, we assume that $f_i$ depends on the entire shared model $y$. See \cref{remark:implementation} for discussions of the general case. } $f_i: \Xcal_i \times \Ycal \to \Real$ that is usually \emph{non-convex}.
Under the standard assumption that agents' measurements are corrupted by independent noise, the global MLE problem takes the following form.

\begin{problem}[Collaborative Geometric Estimation]
\begin{subequations}
	\begin{align}
	\underset{x_i, y}{\min}& \quad \;f(x,y) \triangleq \sum_{i=1}^N f_i(x_i, y),
	\label{eq:problem_formulation_objective} \\
	\;
	\textrm{s.t.}& 
	\quad x_i \in \Xcal_i, \; \forall i \in [N],
	\;\; y \in \Ycal.
	\end{align}
	\label{eq:problem_formulation}
\end{subequations}
\label{problem_formulation}
\end{problem}
\vspace{-0.5cm}

In \eqref{eq:problem_formulation}, we use $x \in \Xcal$ to denote the concatenation of all private variables $x_i, \; i \in [N]$. 
Next, we present several motivational examples related to multi-robot navigation.

\begin{example}[Collaborative Bundle Adjustment]
\normalfont
	Bundle adjustment (BA) \cite{Triggs00BA} is a crucial building block of 
	modern visual SLAM and SfM systems.
	In collaborative BA, agents jointly estimate a global map using local measurements collected by monocular cameras.
	Assuming known camera intrinsics, the private variable $x_i$ contains camera poses of agent $i$, \ie, 
	$x_i = \{T^{(i)}_{1}, \hdots, T^{(i)}_{n_i}\} \in \SE(3)^{n_i}$. 
	The shared variable $y$ consists of points in the global map, \ie,
	$y = \{y_1, y_2, \hdots, y_m\} \in \Real^{3 \times m}$. 
	Under the standard Gaussian noise model, the local cost function 
	is given by the sum of local squared reprojection errors,
	\begin{equation}
		f_i(x_i, y) 
		= \sum_{j = 1}^{n_i}\sum_{l \in L_{ij}} w^{(i)}_{jl} \norm{q^{(i)}_{jl} - \pi(T^{(i)}_{j}, y_l)}^2_2.
		\label{eq:BA_objective}
	\end{equation}
	In \eqref{eq:BA_objective}, 
	$L_{ij} \subseteq [m]$ denotes the set of points observed by agent $i$ at pose 
	$T^{(i)}_j$, 
	$\pi(\cdot, \cdot)$ is the camera projection model, 
	$q^{(i)}_{jl} \in \Real^2$ denotes noisy observation on the image plane, and $w^{(i)}_{jl} > 0$ is the corresponding measurement weight.
	\label{example:BA} 
\end{example}

\begin{example}[Collaborative Point Cloud Registration]
	\normalfont
	Multiple point cloud registration (\eg, \cite{Chaudhury13GlobalPointSets})
	is an important problem with robotic applications such as merging multiple point cloud maps or collaborative SLAM with range sensors.
	In this case, the private and shared variables are the same as \cref{example:BA}.
	The local cost function is given by,
	\begin{equation}
		f_i(x_i, y) 
		= \sum_{j = 1}^{n_i}\sum_{l \in L_{ij}} w^{(i)}_{jl} 
		\norm{y_l - R_j^{(i)} q^{(i)}_{jl} - t_j^{(i)}}^2_2,
		\label{eq:point_cloud_registration_objective}
	\end{equation}
	where $T_j^{(i)} = (R_j^{(i)}, t_j^{(i)})$ denote the rotation matrix and translation vector of the $j$th pose of agent $i$, and $q^{(i)}_{jl} \in \Real^3$ denotes noisy 3D observation in the local frame.
	%Other notations follow \cref{example:BA}.
	\label{example:point_cloud_registration}
\end{example}

\begin{example}[Collaborative Object-Based Pose Graph Optimization]
	\normalfont
	In some applications, it suffices to produce an object-level map of the environment (\eg, \cite{Choudhary17IJRR}). 
	%In this case, robots solve an object-based pose graph optimization.
	In this case, the set of shared variables becomes 
	$y = \{T_1, T_2, \hdots, T_m\} \in \SE(3)^m$, where $T_l$ is the pose of object $l$.
	The local cost function (using chordal distance) is given by,
	\begin{equation}
		f_i(x_i, y) 
		= \sum_{j = 1}^{n_i}\sum_{l \in L_{ij}} w^{(i)}_{jl} 
		\norm{T_l - T^{(i)}_{j} \; \widetilde{T}^{(i)}_{jl}}^2_{\Omega^{(i)}_{jl}},
		\label{eq:pgo_objective}
	\end{equation}
	where $\widetilde{T}^{(i)}_{jl} \in \SE(3)$ is a noisy relative measurement of object $l$
	collected by agent $i$ at pose $T^{(i)}_{j}$, and $\Omega^{(i)}_{jl}$ is the corresponding measurement precision matrix.
	%Other notations follow \cref{example:BA}.
	\label{example:pgo}
\end{example}

In this work, we focus on collaborative BA (\cref{example:BA}) in our experimental validation (\cref{sec:experiments}), due to its fundamental role in multi-robot visual SLAM~\cite{Forster13CSLAM,Schmuck18CCM,Ebadi20LAMP,Schmuck21covins}.
However, we note that our approach extends beyond the above examples 
to many other multi-agent estimation problems that can be described with a factor graph \cite{Dellaert17FactorGraph}.
\optional{
In particular, these include other multi-view reconstruction problems that use alternative sensors or estimate other types of geometric primitives (\eg, planes, quadrics, and cuboids).}

%!TEX root = ../root.tex
\section{Proposed Algorithm}
\label{sec:algorithm}
In this section, we present our communication-efficient distributed algorithm for solving Problem~\ref{problem_formulation}.
In \cref{sec:algorithm:basic_form}, we develop the basic form of our method based on distributed approximate second-order updates.
Similar to DDF-SAM \cite{Cunningham10DDFSAM,Cunningham13DDFSAM2}, in each iteration our method analytically eliminates the updates to private variables, which leads to more effective updates and also protects the privacy of participating agents.
However, unlike DDF-SAM, our method avoids the transmission of dense matrices resulting from elimination, which makes it applicable to larger scale problems.
Furthermore, in \cref{sec:algorithm:lazy}, we augment our basic method with \emph{lazy communication}, which achieves significant communication reduction.
Lastly, \cref{sec:algorithm:pseudocode} summarizes the discussion and presents the complete algorithm. 

\subsection{Distributed Update with Analytic Elimination}
\label{sec:algorithm:basic_form}
At each iteration, agents collaboratively compute an updated solution that decreases the global cost in Problem~\ref{problem_formulation}.
To start, each agent $i$ constructs a second-order approximation $\mhat_i$ for its local cost $f_i$, which is defined at the tangent space of the current iterate $(x_i, y)$.
Intuitively, $\mhat_i$ approximates the true local cost $f_i$ when perturbing $x_i$ and $y$ on the tangent space.
Formally, given tangent vectors $(u_i, v) \in T_{x_i} \Xcal_i \times T_{y} \Ycal$, 
we define\footnote{Note that $\mhat_i$ depends on the linearization points $x_i$ and $y$. We drop this from our notation for simplicity.}
\begin{equation}
	\small
	\begin{aligned}
	& \mhat_i(u_i, v) \triangleq \\
	& \revise{f_i(x_i,y)} 
	\!+\!
	\bigg \langle 
	\underbrace{
		\bmat
		g_{ix} \\
		g_{iy}
		\emat}_{g_i}, 
	\bmat
	u_i \\
	v
	\emat
	\bigg \rangle
	\!+\!
	\frac{1}{2}
	\bigg \langle
	\bmat
	u_i \\
	v
	\emat,
	\underbrace{
		\bmat
		A_{i}      & \!\!\!C_i \\
		{C_{i}}^\top & \!\!\!B_i
		\emat}_{M_i}
	\bmat
	u_i \\
	v
	\emat
	\bigg \rangle.
	\end{aligned}
\label{eq:local_quadratic_model}
\end{equation}
In \eqref{eq:local_quadratic_model}, 
$g_i \triangleq \rgrad f_i(x_i, y)$ is the local Riemannian gradient. 
The user-specified linear map $M_i \succ 0$ serves as an {approximation} of the 
local Riemannian Hessian, and is assumed to be symmetric and positive definite. 
\revise{For geometric estimation problems such as BA \eqref{eq:BA_objective}, we obtain the second-order approximation via the Riemannian Levenberg--Marquardt (LM) method~\cite[Chapter~8]{Absil2009Book}.}
In this case, we have $M_i = {J_i} ^\top J_i + \lambda I$,
where $J_i$ is the Jacobian of agent $i$'s measurement residuals, 
and $\lambda > 0$ is a regularization parameter that ensures $M_i$ to be positive definite.

Given the local approximations $\mhat_i$, a second-order approximation of the global cost $f$ is given by $\mhat(u,v) \triangleq \sum_{i=1}^N \mhat_i(u_i, v)$, 
where we use $u$ to denote the concatenation of local tangent vectors $u_i$.
Note that $\mhat$ can be expanded as,
\begin{equation}
	\small
	\mhat(u,v) = 
	f(x,y) + 
	\bigg \langle
	\underbrace{
	\bmat
	g_x \\ g_y
	\emat}_{g},
	\bmat
	u \\ v
	\emat
	\bigg \rangle
	+ \frac{1}{2}
	\bigg \langle
	\bmat
	u \\
	v
	\emat,
	\underbrace{
		\bmat
		A        & \!\!\!C \\
		{C}^\top & \!\!\!B
		\emat}_{M}
	\bmat
	u \\
	v
	\emat
	\bigg \rangle.
	\label{eq:global_quadratic_model}
\end{equation}
It can be verified that $g = \rgrad f(x, y)$ is the Riemannian gradient of the global objective. 
The linear map $M$ in \eqref{eq:global_quadratic_model} is now an approximation of the global Riemannian Hessian.
More importantly, $M$ is a block matrix with an \emph{arrowhead sparsity pattern}, and its blocks are related to the blocks of $M_i$ in \eqref{eq:local_quadratic_model} as follows,
\begin{align}
	A \!=\! \Diag(A_1, \hdots, A_N), 
	%\nonumber
	B \! = \!\!\! \sum_{i=1}^N B_i,
	{C}^\top \!\!\!\! = \!\! \bmat {C^\top_1} \hdots {C^\top_N} \emat.
	%\ \ \text{and}\ \  
	\label{eq:ABC_def}
\end{align}

In the proposed method, we seek to compute an update for all variables by approximately minimizing $\mhat$.
To proceed, we \emph{analytically eliminate} private vector $u$ from \eqref{eq:global_quadratic_model}. 
Formally, define $\ustar(v) \triangleq \argmin_{u} \mhat(u,v)$ as the optimal private vector conditioned on the shared vector.
Furthermore, define the \emph{reduced second-order approximation} as $\hhat(v) \triangleq \mhat(\ustar(v), v)$, which only involves the shared vector $v$.
Both $\ustar(v)$ and $\hhat(v)$ admit closed-form expressions.
\begin{lemma}[Reduced second-order approximation]
	For each agent $i \in [N]$, 
	the corresponding optimal private vector is,
	\begin{equation}
		\ustar_i(v) = -A_i^{-1}(C_i v + g_{ix}), \;\; \forall i \in [N].
		\label{eq:ustar_def}
	\end{equation}
	Furthermore, $\hhat(v)$ has the closed-form expression,
	\begin{align}
	\!\!\!\!\!\!\!	\hhat(v) = f(x,y) \!-\! \frac{1}{2} \inner{g_x}{A^{-1} g_x}
		\!+\! \inner{w}{v} \! +\! \frac{1}{2} \inner{v}{S v},\label{eq:reduced_global_quadratic_model} 
		\end{align}
		where vector $w$ and matrix $S$ are defined as,
		\begin{align}
		w &\triangleq \sum_{i=1}^N w_i, \ \  w_i \triangleq g_{iy} - C_i^\top A_i^{-1} g_{ix}, \; \forall i \in [N]. \label{eq:w_def} \\
		S &\triangleq \sum_{i=1}^N S_i, \ \ \; S_i \triangleq B_i - C_i^\top A_i^{-1} C_i, 
		\; \forall i \in [N]. \label{eq:S_def}
	\end{align}
	\label{lem:schur_complement}
	\vspace{-0.3cm}
\end{lemma}

In the following, we refer to $w$ in \eqref{eq:w_def} and $S$ in \eqref{eq:S_def} as the \emph{reduced gradient} and \emph{reduced Hessian}, respectively.
\remove{For nonlinear least squares problems, Lemma~\ref{lem:schur_complement} is commonly known as the Schur complement trick, and is widely used in existing centralized  \cite{gtsam,g2o,ceres-solver} and distributed solvers \cite{Cunningham10DDFSAM,Cunningham13DDFSAM2}. }
\revise{The analytic elimination technique presented above has been widely used to solve SLAM and BA \cite{Triggs00BA}, and is a special case of the \emph{variable projection} approach to solve nonlinear least squares problem \cite{OLeary2013VariableProjection}. In the distributed setting, }
Lemma~\ref{lem:schur_complement}
suggests that the server can first aggregate $w_i$ and $S_i$ from all agents, 
and then minimize $\hhat(v)$ by computing $\vstar = - S^{-1} w$. 
This type of approach has been proposed by DDF-SAM \cite{Cunningham10DDFSAM,Cunningham13DDFSAM2}.
Nevertheless, for large-scale problems such as BA, this approach is less suitable as it requires the communication of the $S_i$ matrices, which are generally \emph{dense} and thus expensive to evaluate, store, and transmit. 

To design a communication-efficient update, we instead resort to finding an {approximate} minimizer of $\hhat(v)$.
In the following, let $k$ denote the iteration number.
We let our approximate minimizer of $\hhat(v)$ take the following form,
\begin{equation}
	v^k \triangleq -\gamma P^k w^k,
	\label{eq:server_update}
\end{equation}
where $\gamma > 0$ is a constant stepsize, and $P^k$ is a \emph{sparse} matrix that approximates the inverse of the reduced Hessian $S^k$.
Viewing $w^k$ as the reduced gradient and $P^k$ as a preconditioner, 
we may interpret \eqref{eq:server_update} as a single step of preconditioned Riemannian gradient descent.
\optional{Note that our method is agnostic to the specific choice of preconditioners.}
In the following, we use the block Jacobi preconditioner \cite{Golub96Matrix} due to its simplicity,
\begin{equation}
	P^k = \left(\sum_{i=1}^N D_i^k\right)^{-1}, \;\; D_i^k \triangleq \Diag(S^k_{i,1}, \hdots, S^k_{i,m}), 
	\label{eq:jacobi}
\end{equation}
where $S^k_{i,l}$ is the $l$-th diagonal block of $S^k_i$.
Note that each block $l$ corresponds to a single element $y_l$ in the shared variable $y$ (see \cref{sec:problem_formulation}). 
With \eqref{eq:server_update} and \eqref{eq:jacobi}, agent $i$ only needs to upload $w^k_i$ and the diagonal blocks of $S_i^k$ to the server. 
Furthermore, the server can easily compute $P^k$ as it only requires inverting a block-diagonal matrix.

Once the shared update $v^k$ is computed, we leverage Lemma~\ref{lem:schur_complement} to compute the corresponding optimal second-order update for each agent's private variable, 
\begin{equation}
	u_i^k \triangleq u_i^\star(v^k), \; \forall i \in [N].
	\label{eq:agent_update}
\end{equation}
Finally, we compute the updated estimates using retraction,
\begin{equation}
	y^{k+1} = \Retr_{y^k}(v^k), \; x^{k+1}_i = \Retr_{x^k_i} (u_i^k), \forall i \in [N],
	\label{eq:retraction}
\end{equation}
and the algorithm proceeds to the next iteration.

%%%%%%%%%%%%%%%%%%%%%%%%%%%%%%%%%%%%%%%%%%%%%%%%%%%%%%%%%%%%%%%%%%%%%%%%%%%%%%%%%%%%%%%%%%%%%%%%%%%%%%%%%%%%%%%%%%%%%%%%%%%%%%%%%%%%%%%%%%%%%%%%%%%%%%%%%%%%%%%%%%%%%%%%%%%%%%%%%%%%%%%%%%%%%%%%%%%%%%%%%%%%%%%%%%%%%%%%%%%%%%%%%%%%%%%%%%%%%%%%%%%%%%%%%%%%%%%%%%%%%%%%%%%%%%%%%%%%%%%%%%%%%%%%%%%%%%%%%%%%%%%%%%%%%%%%%%%%%%%%%%%%%%%%%%%%%%%%%%%%%%%%%%%%%%%%%%%%%%%%%%%%%%%%%%%%%%%%%%%%%%%%%%%%%%%%%%%%%%%%%%%%%%%%%%%%%%%%%%%%
\subsection{Incorporating Lazy Communication}
\label{sec:algorithm:lazy}
In the method developed so far, at each iteration, 
agent $i$ needs to upload $w_i$ and $D_i$ to the server.
For larger problems, the resulting transmission can still become too expensive.
In this subsection, we present a technique to further reduce communication.
The core idea behind our approach is \emph{lazy communication}: 
when some blocks of $w_i$ and $D_i$ do not change significantly from previous iterations, agent $i$ simply skips the transmission of those blocks, and the server reuses values received at previous iterations for its computation. 
In the following, we describe this process in detail for the computation of preconditioner and the reduced gradient, respectively.
Without loss of generality, we present our method from the perspective of agent $i$.

\myParagraph{Lazy communication of preconditioner} 
Let $k$ be the current iteration number. 
For each block $l$, let $k' < k$ be the last iteration when agent $i$ uploads $S^{k'}_{i,l}$ to the server.\footnote{For notation simplicity, we drop the dependence of $k'$ on $i$ and $l$.}
Using $S^{k'}_{i,l}$, we can compute an approximation of $S^k_{i,l}$ as, 
\begin{equation}
	\widetilde{S}^k_{i,l} \triangleq 
	\transporter{k'}{k}_l \circ S^{k'}_{i,l} \circ \transporter{k}{k'}_l,
	\label{eq:apx_jacobi_block}
\end{equation}
where $\transporter{k}{k'}_l$ is the matrix that represents a \emph{transporter}~\cite[Sec.~10.5]{Boumal20Book} from the tangent space at iteration $k$ to iteration $k'$, and $\transporter{k'}{k}_l$ is its adjoint. 
Intuitively, transporters are needed to ensure that the approximation defined in \eqref{eq:apx_jacobi_block} represents a valid linear map on the tangent space at iteration $k$. 
\revise{For matrix manifolds, a simple and computationally efficient transporter can be obtained from orthogonal projections to tangent spaces \cite[Proposition~10.60]{Boumal20Book}.}
Note that since \eqref{eq:apx_jacobi_block} only uses past information, both the server and agent $i$ can compute $\widetilde{S}^k_{i,l}$ \emph{without any communication}.

The above approximation leads to the following lazy communication scheme.
First, agent $i$ compares $S^k_{i,l}$ and its approximate version $\widetilde{S}^k_{i,l}$ locally. 
Then, agent $i$ only uploads $S^k_{i,l}$ to the server if the approximation error is large,
\begin{equation}
	\norm{S_{i,l}^k - \widetilde{S}^k_{i,l}} > \delta_p \norm{S_{i,l}^k},
	\label{eq:lazy_jacobi_rule}
\end{equation}
where $\delta_p \geq 0$ is a user defined threshold. 
On the other hand, if \eqref{eq:lazy_jacobi_rule} does not hold (\ie, approximation error is small), agent $i$ skips the communication of $S^k_{i,l}$, and the server uses the approximation $\widetilde{S}^k_{i,l}$ instead.

In summary, for each agent $i$, instead of using $D^k_i$ as defined in \eqref{eq:jacobi}, the server now uses an approximation $\widehat{D}^k_i$ that consists of a mixture of exact and approximate blocks. Specifically, the $l$-th diagonal block of $\widehat{D}^k_i$ is given by,
\begin{equation}
	\widehat{D}^k_{i,l} \triangleq \begin{cases}
		S^k_{i,l}, \;\; \text{if \eqref{eq:lazy_jacobi_rule} holds,} \\
		\widetilde{S}^k_{i,l}, \;\; \text{otherwise.}
	\end{cases}
	\label{eq:Dhat_i_def}
\end{equation}
Finally, the preconditioner becomes
$P^k = (\sum_{i=1}^N \widehat{D}^k_i)^{-1}$.

\myParagraph{Lazy communication of reduced gradient} 
We can employ a similar strategy to design a lazy communication rule for the transmission of the reduced gradient $w_i$, which is needed by the server to compute the update step in \eqref{eq:server_update}.
For this purpose, let us view $w_i$ as a block vector, where each block $w_{i,l}$ corresponds to a single element $y_l$ in the shared variable. 
For each block $l$, let $k' < k$ be the last iteration when agent $i$ uploads $w^{k'}_{i,l}$ to the server.
Using $w^{k'}_{i,l}$, we can compute an approximation of $w^k_{i,l}$ as follows,
\begin{equation}
	\widetilde{w}^k_{i,l} \triangleq \transporter{k'}{k}_l \left( w^{k'}_{i,l} \right).
	\label{eq:apx_reduced_grad_block}
\end{equation}
Once again, a transporter is needed to ensure \eqref{eq:apx_reduced_grad_block} defines a valid tangent vector on the tangent space at the current iteration $k$.
\remove{Since $\widetilde{w}^k_{i,l}$ only involves past information, it can be computed by the server and agent $i$ without any communication.}
\revise{Similar to \eqref{eq:apx_jacobi_block}, computing \eqref{eq:apx_reduced_grad_block} does not require communication between the server and agent $i$.}

Similar to the previous development, at each iteration, agent $i$ only uploads $w^k_{i,l}$ to the server if it differs significantly from $\widetilde{w}^k_{i,l}$. 
Specifically, we define the following communication triggering condition,
\begin{equation}
	\norm{\wtilde^k_{i,l} - w^k_{i,l}}^2_{P^k_l} > \frac{1}{m N^2} \sum_{d=1}^{\bar{d}} \epsilon_d \norm{\what^{k-d}}^2_{P^{k-d}}.
	\label{eq:lazy_reduced_grad_rule}
\end{equation}
The left hand side of \eqref{eq:lazy_reduced_grad_rule} measures the approximation error, using the norm associated with the current preconditioner of this block. 
The right hand side defines a threshold on the approximation error using information from the past $\bar{d}$ iterations. 
Specifically, $\what^{k-d}$ is the approximate reduced gradient used by the server at iteration $k-d$, and its exact definition is provided in \eqref{eq:what_def}.
Both $\bar{d}$ and weights $\{\epsilon_d \geq 0, d = 1, \hdots \bar{d}\}$ are user specified constants.
Note that setting $\epsilon_d = 0$ forces the agent to always upload.

Consequently, on the server's side,
instead of using the up-to-date $w_i^k$ vector for agent $i$, it uses an approximate version $\what^k_i$ that consists of both exact and approximate blocks,
\begin{equation}
	\what^k_{i,l} \triangleq \begin{cases}
		w^k_{i,l}, \;\; \text{if \eqref{eq:lazy_reduced_grad_rule} holds,} \\
		\wtilde^k_{i,l}, \;\; \text{otherwise.}
	\end{cases}
	\label{eq:what_i_def}
\end{equation}
Finally, instead of using $w^k$ to compute the update in \eqref{eq:server_update}, the server uses the approximation defined as,
\begin{equation}
	\what^k \triangleq \sum_{i=1}^N \what_i^k.
	\label{eq:what_def}
\end{equation}
We conclude this subsection by noting that the lazy communication condition for reduced gradient \eqref{eq:lazy_reduced_grad_rule} is more complex than the condition for preconditioner \eqref{eq:lazy_jacobi_rule}. 
The more complex rule \eqref{eq:lazy_reduced_grad_rule} is needed for our convergence analysis, and we provide more discussions in \cref{sec:convergence}.

\subsection{The Complete Algorithm}
\label{sec:algorithm:pseudocode}

\begin{algorithm}[t]
	\caption{\textsc{\Algorithm} }
	\label{alg:lazy}
	\begin{algorithmic}[1]
		\small 
		\State Initialize solution $x^0, y^0$. 
		\For{iteration $k = 0, 1, \hdots $}
			\State In parallel, agent $i$ computes the local second-order approximation $\mhat_i$~\eqref{eq:local_quadratic_model}, and evaluates $w^k_i$ \eqref{eq:w_def} and $D^k_i$ \eqref{eq:jacobi}
			\State {\small \textcolor{green!50!black}{// Lazy communication of preconditioner}} \label{alg:lazy:jacobi_start}
			\For{each agent $i$ \textbf{in parallel}}
				\For{each block $l$}
					\State Upload $S_{i,l}^k$ to server if \eqref{eq:lazy_jacobi_rule} is true
				\EndFor
			\EndFor
			\State Server collects uploads and forms $\widehat{D}^k_i$ \eqref{eq:Dhat_i_def} for each agent
			\State Server computes preconditioner $P^k = (\sum_{i=1}^N \widehat{D}^k_i)^{-1}$ and broadcasts to agents \label{alg:lazy:jacobi_end}
			\State {\small \textcolor{green!50!black}{// Lazy communication of reduced gradient}} \label{alg:lazy:reduced_grad_start}
			\For{each agent $i$ \textbf{in parallel}}
				\For{each block $l$}
				\State Upload $w_{i,l}^k$ to server if \eqref{eq:lazy_reduced_grad_rule} is true
				\EndFor
			\EndFor
			\State Server collects uploads and forms $\widehat{w}^k_i$ and $\what^k$ \label{alg:lazy:reduced_grad_end}
			\State {\small \textcolor{green!50!black}{// Compute update vector and next iterate}} \label{alg:lazy:update_start}
			\State Server computes $v^k = -\gamma P^k \what^k$ and broadcasts to agents
			\State In parallel, each agent computes $u^k_i = \ustar_i(v^k)$; \cf \eqref{eq:ustar_def}
			\State Both server and agents update iterates 
			$$ \small y^{k+1} = \Retr_{y}(v^k), \; x^{k+1}_i = \Retr_{x^k_i}(u_i^k).$$\label{alg:lazy:update_end}
			\vspace{-0.5cm}
		\EndFor
	\end{algorithmic}
\end{algorithm}

We collect the steps discussed above and present the Lazily Aggregated Reduced Preconditioned Gradient (\Algorithm) algorithm with the 
complete pseudocode in Algorithm~\ref{alg:lazy}.
Each iteration of \Algorithm has three stages.
The first stage (lines~\ref{alg:lazy:jacobi_start}-\ref{alg:lazy:jacobi_end}) performs the lazy communication of the preconditioner. 
The second stage (lines~\ref{alg:lazy:reduced_grad_start}-\ref{alg:lazy:reduced_grad_end}) performs the lazy communication of the reduced gradients.
We note that this stage needs to happen after the first stage, as the triggering rule for the reduced gradient \eqref{eq:lazy_reduced_grad_rule} depends on the preconditioner $P^k$.
The third stage (lines~\ref{alg:lazy:update_start}-\ref{alg:lazy:update_end}) uses the lazily aggregated information to compute the next iterate of the algorithm.

\begin{remark}[Novelty with respect to \cite{Chen18LAG}]
	\normalfont
	Our lazy communication scheme is inspired by Chen~\etal~\cite{Chen18LAG}, who study lazily aggregated gradient methods in distributed optimization.
	However, our algorithm and analysis (\cref{sec:convergence}) consists of the following important innovations to account for the unique challenges of Problem~\ref{problem_formulation}: 
	(i) we consider problems with \emph{non-convex} manifold constraints that are prevalent in robot perception applications, 
	(ii) we incorporate the use of approximate second-order updates that require substantial changes in the convergence analysis, 
	(iii) we handle private variables via analytic elimination,
	and (iv) we propose lazy communication on individual blocks of the gradient and preconditioner, which leads to further communication reduction. 
	\label{remark:novelty}
\end{remark}

\begin{remark}[Implementation]
    \normalfont
    In many applications, such as collaborative SLAM, each agent $i$ only observes parts of the shared model during navigation.
    Consequently, the local cost $f_i$ only depends on the observed subset of the shared variable $y$.
    In our implementation and experiments (\cref{sec:experiments}), 
    we account for this fact by performing lazy communication \emph{only} on the observed parts of $y$ for each agent.
    \label{remark:implementation}
\end{remark}

\begin{figure*}[ht!]
	\centering
	\begin{subfigure}[t]{0.19\linewidth}
		\includegraphics[width=\textwidth]{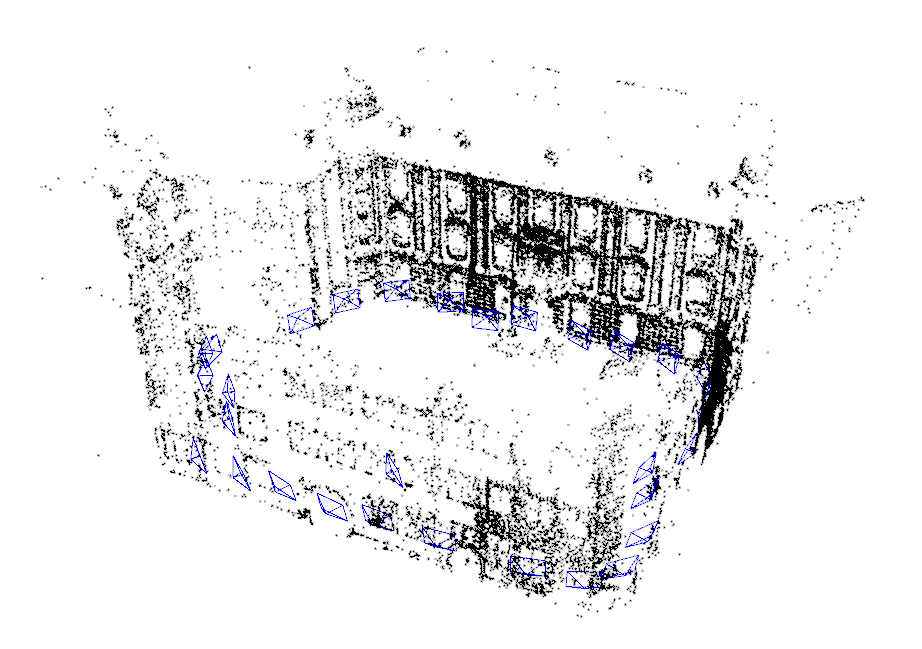}
		\caption{\textsf{Castle30} dataset}
		\label{fig:castle30:dataset}
	\end{subfigure}
	\hfill
	\begin{subfigure}[t]{0.19\linewidth}
		\includegraphics[width=\textwidth]{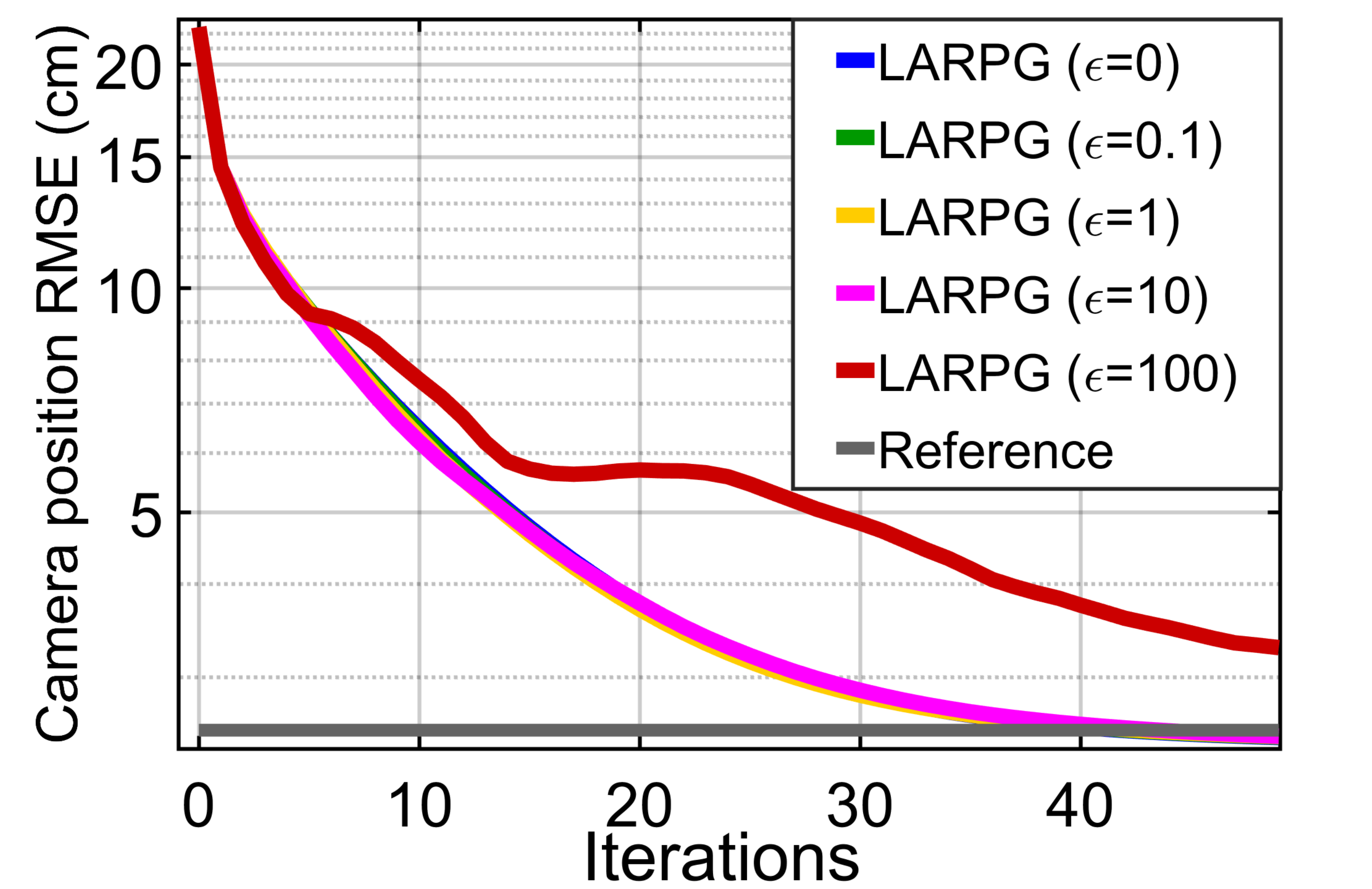}
		\caption{Effect of varying $\epsilon$ (RMSE vs. iteration)}
		\label{fig:castle30:cam_error_vs_iter}
	\end{subfigure}	
	\hfill
	\begin{subfigure}[t]{0.19\linewidth}
		\includegraphics[width=\textwidth]{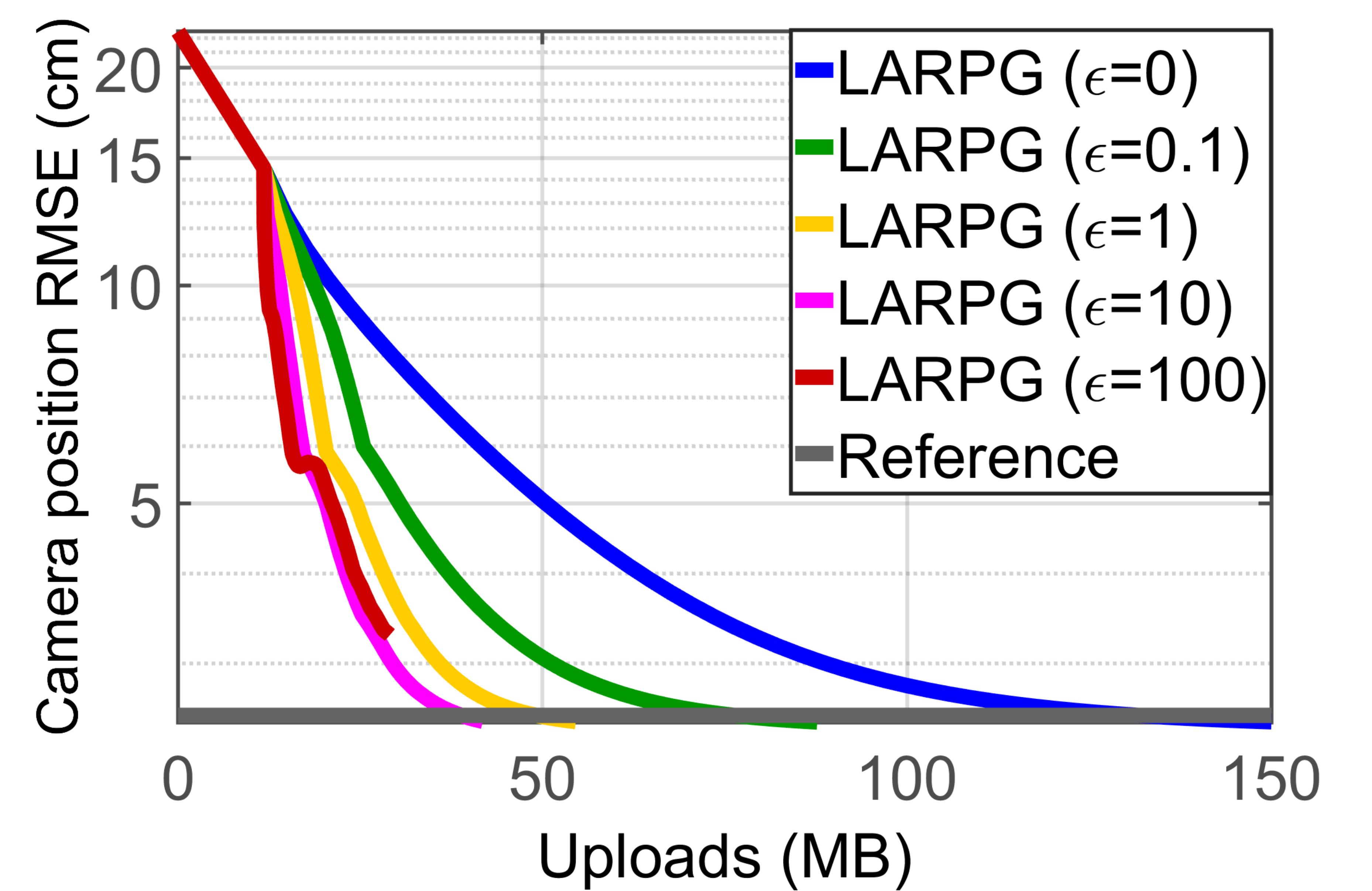}
		\caption{Effect of varying $\epsilon$ (RMSE vs. total uploads)}
		\label{fig:castle30:cam_error_vs_comm}
	\end{subfigure}
	\hfill
	\begin{subfigure}[t]{0.19\linewidth}
		\includegraphics[width=\textwidth]{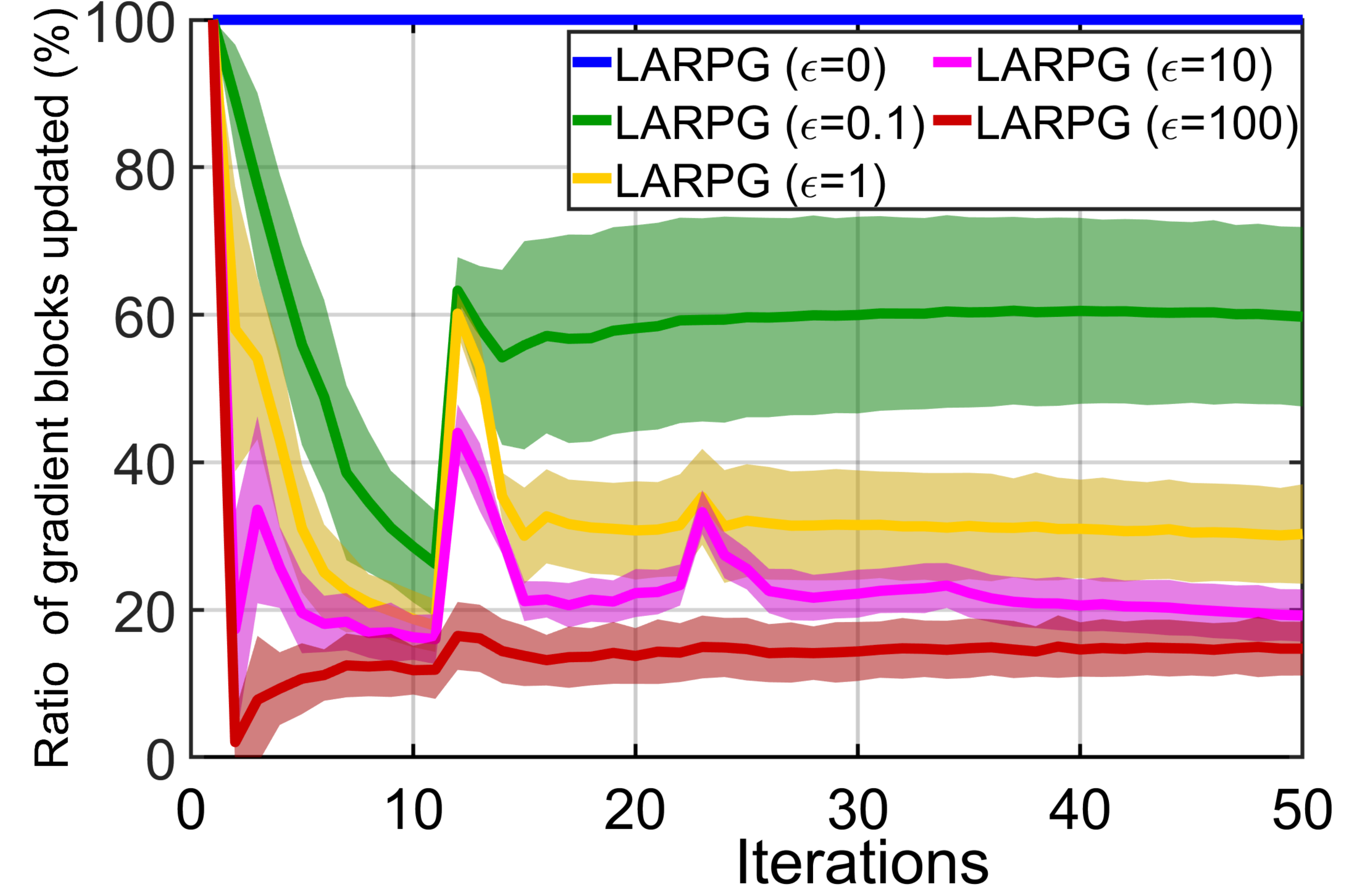}
		\caption{Effect of varying $\epsilon$ (\% gradient blocks uploaded)}
		\label{fig:castle30:lazy_comm_pattern}
	\end{subfigure}
	\hfill
	\begin{subfigure}[t]{0.19\linewidth}
		\includegraphics[width=\textwidth]{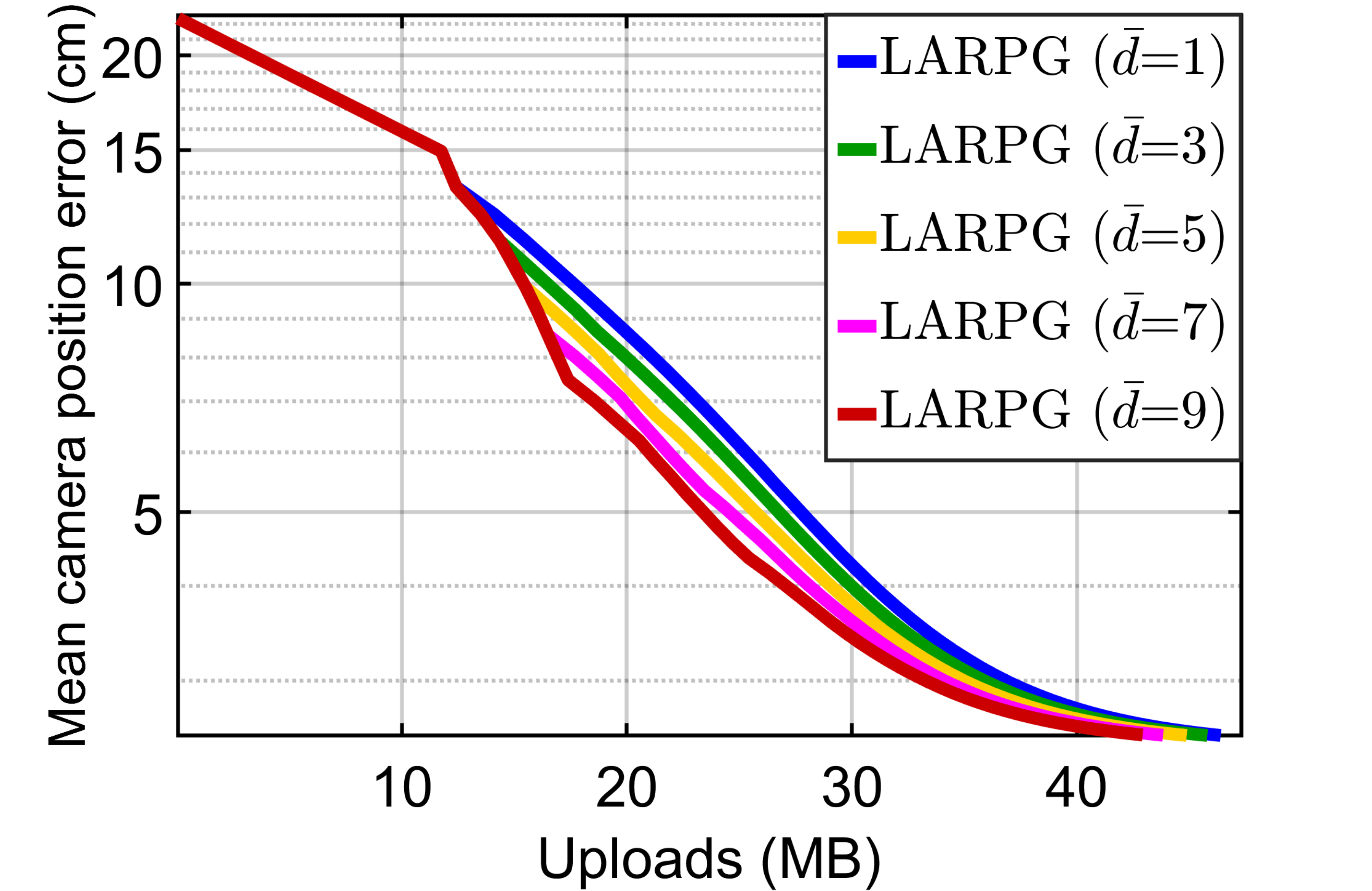}
		\caption{Effect of varying $\bar{d}$ (RMSE vs. total uploads)}
		\label{fig:castle30:vary_queue_size}
	\end{subfigure}	
	\vspace{-0.2cm}
	\caption{Evaluation of lazy communication on \textsf{Castle30} dataset \cite{Strecha08Benchmark}.
	We evaluate the performance of \Algorithm under varying values of parameters $\epsilon$ and $\bar{d}$ that control the behavior of lazy communication.
	} \vspace{-0.3cm}
	\label{fig:castle30}
\end{figure*}

%!TEX root = ../root.tex
\section{Convergence Analysis}
\label{sec:convergence}
Since \Algorithm allows agents to lazily upload information to the server, 
it is unclear if the algorithm can converge to a desired solution in general. 
In this section, we provide a rigorous answer to this important question.
In particular, we show that under mild technical conditions, \Algorithm provably converges to a first-order critical point of Problem~\ref{problem_formulation}, despite the use of lazy communication. 
Assumption~\ref{as:assumptions} below summarizes these technical assumptions.
\revise{
	\begin{assumption}
		There exist constants $L > \mu \geq c_g > 0$ and $\mu_p, \sigmabar > 0$ such that the following conditions hold at any iteration $k \in \Natural$ of \cref{alg:lazy},
		\begin{enumerate}[start=1,label={A\arabic*}]
			\item (Lipschitz-type gradient for pullbacks \cite{Boumal2018Convergence}) 
			Let $f^k \triangleq f(x^k,y^k)$ and $g^k \triangleq \rgrad f(x^k,y^k)$ denote the objective and Riemannian gradient at iteration $k$.
			The \emph{pullback function} $\fhat^k(u,v) \triangleq f(\Retr_{x^k}(u), \Retr_{y^k}(v))$ satisfies
			\begin{equation}
				\!\!\!\!\!\!\!\!\!\!\left | \fhat^k(u,v) - 
				\left [f^k + 
				\left \langle
				\bmat
				g^k_x \\ g^k_y
				\emat,
				\bmat
				u \\ v
				\emat
				\right \rangle
				\right ]
				\right |
				\leq
				\frac{c_g}{2}
				\norm{\bmat
					u \\ v
					\emat}^2,
				\label{eq:lipschitz_type_gradient}
			\end{equation}
			for all $(u,v) \in T_{x^k} \Xcal \times T_{y^k} \Ycal$.
			%\vspace{-0.3cm}
			\label{as:lipschitz_type_gradient}
			\item (Bounded Hessian approximation)
			The approximate Hessian $M^k$ at iteration $k$ satisfies $\mu I \preceq M^k \preceq L I$.
			\label{as:hessian_apx_bound}
			\item (Preconditioner) 
			The preconditioner $P^k$ at iteration $k$ satisfies $P^k \succeq \mu_p I$ and $\norm{S^kP^k}_{P^k} \leq \sigmabar$.
			\label{as:preconditioner}
		\end{enumerate}
		\label{as:assumptions}
	\end{assumption}
}
Above, \eqref{as:lipschitz_type_gradient} is first introduced in \cite{Boumal2018Convergence} as a generalization of the standard Lipschitz smoothness assumption to the Riemannian setting.
\optional{Intuitively, \eqref{as:lipschitz_type_gradient} bounds the pullback function by its local linearization. }
Prior work in distributed BA \cite{Eriksson16ConsensusBA,Zhang17DistBA} requires similar smoothness conditions for convergence.
In comparison, our assumptions and convergence guarantees extend beyond BA and hold for more general problems.
\eqref{as:hessian_apx_bound} assumes that the employed Hessian approximation $M$ is bounded, which is also a standard assumption. 
Lastly, \eqref{as:preconditioner} assumes that the preconditioner $P^k$ is sufficiently positive definite, and the approximation error of $P^k$ as the inverse of the reduced Hessian $S^k$ is bounded.
We note that the latter two assumptions \eqref{as:hessian_apx_bound} and \eqref{as:preconditioner} can be easily satisfied, since the user has freedom to change what Hessian approximation $M$ and preconditioner $P$ to use; for example, $M = c_g I$ is always a valid choice that satisfies \eqref{as:hessian_apx_bound}.

\optional{
The key to our convergence analysis (inspired by \cite{Chen18LAG}) is to study the iterates of \Algorithm with respect to a \emph{Lyapunov function},
\begin{equation}
	V^k \triangleq f(x^k, y^k) + \sum_{d=1}^{\bar{d}} \beta_d \norm{\what^{k-d}}^2_{P^{k-d}}.
	\label{eq:lyapunov_paper}
\end{equation}
At iteration $k$, $V^k$ combines the cost function with weighted squared norms of past approximate reduced gradients. 
Intuitively, these squared norms account for the approximation errors caused by lazy communication.
In Appendix~\ref{sec:proof_convergence}, we show that \Algorithm is a \emph{descent} method with respect to $V^k$.
This enables us to establish our main theoretical result.
}

% The following theorem states our main convergence result. 
\begin{theorem}
	Under Assumption~\ref{as:assumptions},
	there exist suitable choices of algorithm parameters
	$\gamma, \bar{d}$, and $\{\epsilon_d \geq 0, d = 1, \hdots, \bar{d}\}$ such that
	after $K$ iterations, the iterates generated by \cref{alg:lazy} satisfy,
	\begin{equation}
	\min_{k \in [K]} \norm{\rgrad f(x^k, y^k)}^2 = O(1/K). 
	\label{eq:convergence}
	\end{equation}
	\label{thm:convergence}
\end{theorem}
\vspace{-0.5cm}
In Appendix~\ref{sec:proof_convergence}, we prove \cref{thm:convergence} and provide explicit parameter settings that guarantee convergence.
The established $O(1/K)$ convergence rate matches standard global convergence result in Riemannian optimization \cite{Boumal2018Convergence}.
While our convergence conditions involve additional parameters, experiments (\cref{sec:experiments}) show that \Algorithm is not sensitive to these parameters and converges under a wide range of values.

%!TEX root = ../root.tex

\section{Experimental Results}
\label{sec:experiments}

In this section, we evaluate \Algorithm on BA problems from benchmark collaborative SLAM and SfM datasets.
All algorithms are implemented in C++ using \software{g2o} \cite{g2o}, 
and experiments are conducted on a computer with Intel i7-7700K CPU and 16 GB RAM.
Unless otherwise mentioned, the default parameters we use for \Algorithm are summarized in
\cref{tab:default_parameters} in the appendix.
Our results show that \Algorithm converges under a wide range of parameter settings,
and compares favorably against existing methods while achieving up to 78\%
total communication reduction. 
In the rest of this section, we first perform ablation studies on the proposed lazy communication scheme (\cref{sec:experiments:ablation}).
Then, we present evaluation and comparison results on large-scale benchmark datasets (\cref{sec:experiments:slam} and \ref{sec:experiments:sfm}). 

% \usepackage{multirow}
% \usepackage{graphicx}
\begin{table*}[ht]
\caption{\small
Evaluation on collaborative SLAM scenarios \cite{Burri16Euroc,Geiger12KITTI}.
Columns \textbf{N, \#IM, \#MP, \#Obs} denote the total number of agents, images (keyframes), map points, and observations, respectively.
\textbf{Init}: input to all algorithms.
\textbf{Ref}: reference solution from centralized optimization using \software{g2o} \cite{g2o}.
\textbf{PCG}: baseline distributed preconditioned conjugate gradient method \cite{Wu11MulticoreBA}.
\textbf{DR}: baseline Douglas-Rachford splitting method \cite{Eriksson16ConsensusBA}.
\textbf{\Algorithm}: proposed method ($\epsilon = 1$).
For each metric, the best-performing distributed method is highlighted in bold.}
\label{tab:cslam}
\resizebox{\textwidth}{!}{%
\begin{tabular}{|l|c|c|c|c|ccccc|ccccc|ccc|}
\hline
\multirow{2}{*}{Dataset} &
  \multirow{2}{*}{$N$} &
  \multirow{2}{*}{\#IM} &
  \multirow{2}{*}{\#MP} &
  \multirow{2}{*}{\#Obs} &
  \multicolumn{5}{c|}{Absolute Trajectory Error (ATE) {[}m{]}} &
  \multicolumn{5}{c|}{Mean Reprojection Error {[}px{]}} &
  \multicolumn{3}{c|}{Total Uploads {[}MB{]}} \\ \cline{6-18} 
 &
   &
   &
   &
   &
  \multicolumn{1}{c|}{Init} \rule{0pt}{8pt} &
  \multicolumn{1}{c|}{Ref} &
  \multicolumn{1}{c|}{PCG} &
  \multicolumn{1}{c|}{DR} &
  \textbf{LARPG} &
  \multicolumn{1}{c|}{Init} &
  \multicolumn{1}{c|}{Ref} &
  \multicolumn{1}{c|}{PCG} &
  \multicolumn{1}{c|}{DR} &
  \textbf{LARPG} &
  \multicolumn{1}{c|}{PCG} &
  \multicolumn{1}{c|}{DR} &
  \textbf{LARPG} \\ \hline
Vicon Room 1 &
  3 &
  464 &
  13K &
  121K &
  \multicolumn{1}{c|}{0.213} &
  \multicolumn{1}{c|}{0.127} &
  \multicolumn{1}{c|}{0.127} &
  \multicolumn{1}{c|}{0.127} &
  \textbf{0.126} &
  \multicolumn{1}{c|}{47.3} &
  \multicolumn{1}{c|}{1.38} &
  \multicolumn{1}{c|}{1.39} &
  \multicolumn{1}{c|}{1.40} &
  \textbf{1.38} &
  \multicolumn{1}{c|}{34} &
  \multicolumn{1}{c|}{26} &
  \textbf{11} \\ \hline
Vicon Room 2 &
  3 &
  631 &
  20K &
  176K &
  \multicolumn{1}{c|}{0.191} &
  \multicolumn{1}{c|}{0.087} &
  \multicolumn{1}{c|}{0.089} &
  \multicolumn{1}{c|}{\textbf{0.088}} &
  \textbf{0.088} &
  \multicolumn{1}{c|}{45.3} &
  \multicolumn{1}{c|}{1.42} &
  \multicolumn{1}{c|}{1.51} &
  \multicolumn{1}{c|}{1.46} &
  \textbf{1.43} &
  \multicolumn{1}{c|}{43} &
  \multicolumn{1}{c|}{32} &
  \textbf{14} \\ \hline
Machine Hall &
  5 &
  719 &
  19K &
  187K &
  \multicolumn{1}{c|}{0.297} &
  \multicolumn{1}{c|}{0.274} &
  \multicolumn{1}{c|}{0.253} &
  \multicolumn{1}{c|}{\textbf{0.215}} &
  0.232 &
  \multicolumn{1}{c|}{50.3} &
  \multicolumn{1}{c|}{1.38} &
  \multicolumn{1}{c|}{3.72} &
  \multicolumn{1}{c|}{\textbf{1.38}} &
  1.43 &
  \multicolumn{1}{c|}{61} &
  \multicolumn{1}{c|}{46} &
  \textbf{17} \\ \hline
KITTI 00 &
  10 &
  1699 &
  96K &
  553K &
  \multicolumn{1}{c|}{6.83} &
  \multicolumn{1}{c|}{5.88} &
  \multicolumn{1}{c|}{5.88} &
  \multicolumn{1}{c|}{\textbf{5.86}} &
  5.87 &
  \multicolumn{1}{c|}{133.1} &
  \multicolumn{1}{c|}{1.08} &
  \multicolumn{1}{c|}{1.49} &
  \multicolumn{1}{c|}{\textbf{1.09}} &
  1.10 &
  \multicolumn{1}{c|}{176} &
  \multicolumn{1}{c|}{133} &
  \textbf{71} \\ \hline
KITTI 06 &
  10 &
  422 &
  22K &
  120K &
  \multicolumn{1}{c|}{10.87} &
  \multicolumn{1}{c|}{10.32} &
  \multicolumn{1}{c|}{10.42} &
  \multicolumn{1}{c|}{\textbf{10.32}} &
  10.36 &
  \multicolumn{1}{c|}{107.9} &
  \multicolumn{1}{c|}{1.11} &
  \multicolumn{1}{c|}{\textbf{1.11}} &
  \multicolumn{1}{c|}{1.12} &
  \textbf{1.11} &
  \multicolumn{1}{c|}{44} &
  \multicolumn{1}{c|}{34} &
  \textbf{16} \\ \hline
\end{tabular}%
}
%\end{table*}
% \usepackage{multirow}
% \usepackage{graphicx}
%\begin{table*}[ht]
	\centering
	\caption{\small 
		Evaluation on collaborative SfM scenarios \cite{Wilson141dsfm}. 
		Each dataset is divided to simulate 50 agents.
		For \Algorithm, we set $\epsilon = 10$.
		All columns are named in the same way as \cref{tab:cslam}.
		For each metric, the best performing distributed method is highlighted in bold.
	}
	\label{tab:1dsfm}
	\resizebox{\textwidth}{!}{%
		\begin{tabular}{|l|c|c|c|ccccc|ccc|ccc|}
			\hline
			\multirow{2}{*}{Dataset} & \multirow{2}{*}{\#IM} & \multirow{2}{*}{\#MP} & \multirow{2}{*}{\#Obs} & \multicolumn{5}{c|}{Mean Reprojection Error [px]}                                                                                   & \multicolumn{3}{c|}{Total Uploads [MB]}                              & \multicolumn{3}{c|}{Average Local Iteration Time [ms]}              \\ \cline{5-15} 
			&                      &                      &                       & \multicolumn{1}{c|}{Init} \rule{0pt}{8pt}& \multicolumn{1}{c|}{Ref}  & \multicolumn{1}{c|}{PCG}           & \multicolumn{1}{c|}{DR}   & \textbf{LARPG}         & \multicolumn{1}{c|}{PCG}  & \multicolumn{1}{c|}{DR}   & \textbf{LARPG}        & \multicolumn{1}{c|}{PCG}         & \multicolumn{1}{c|}{DR}  & \textbf{LARPG} \\ \hline
			Alamo                    & 576                  & 138K                 & 813K                  & \multicolumn{1}{c|}{2.56} & \multicolumn{1}{c|}{1.39} & \multicolumn{1}{c|}{1.57}          & \multicolumn{1}{c|}{1.63} & \textbf{1.44} & \multicolumn{1}{c|}{989}  & \multicolumn{1}{c|}{745}  & \textbf{186} & \multicolumn{1}{c|}{\textbf{16}} & \multicolumn{1}{c|}{151} & 76    \\ \hline
			Ellis Island             & 234                  & 22K                  & 86K                   & \multicolumn{1}{c|}{5.30} & \multicolumn{1}{c|}{2.61} & \multicolumn{1}{c|}{5.04}          & \multicolumn{1}{c|}{4.07} & \textbf{3.24} & \multicolumn{1}{c|}{117}  & \multicolumn{1}{c|}{90}   & \textbf{22}  & \multicolumn{1}{c|}{\textbf{1}}  & \multicolumn{1}{c|}{8}   & 8     \\ \hline
			Gendarmenmarkt           & 704                  & 78K                  & 271K                  & \multicolumn{1}{c|}{4.34} & \multicolumn{1}{c|}{2.02} & \multicolumn{1}{c|}{2.96}          & \multicolumn{1}{c|}{2.67} & \textbf{2.23} & \multicolumn{1}{c|}{379}  & \multicolumn{1}{c|}{286}  & \textbf{73}  & \multicolumn{1}{c|}{\textbf{5}}  & \multicolumn{1}{c|}{35}  & 27    \\ \hline
			Madrid Metropolis        & 345                  & 45K                  & 198K                  & \multicolumn{1}{c|}{3.77} & \multicolumn{1}{c|}{1.28} & \multicolumn{1}{c|}{\textbf{1.48}} & \multicolumn{1}{c|}{1.87} & 1.49          & \multicolumn{1}{c|}{272}  & \multicolumn{1}{c|}{205}  & \textbf{56}  & \multicolumn{1}{c|}{\textbf{3}}  & \multicolumn{1}{c|}{23}  & 19    \\ \hline
			Montreal Notre Dame      & 459                  & 152K                 & 811K                  & \multicolumn{1}{c|}{3.05} & \multicolumn{1}{c|}{1.96} & \multicolumn{1}{c|}{\textbf{2.04}} & \multicolumn{1}{c|}{2.10} & 2.08          & \multicolumn{1}{c|}{1048} & \multicolumn{1}{c|}{790}  & \textbf{171} & \multicolumn{1}{c|}{\textbf{16}} & \multicolumn{1}{c|}{124} & 80    \\ \hline
			Notre Dame               & 548                  & 225K                 & 1180K                 & \multicolumn{1}{c|}{3.97} & \multicolumn{1}{c|}{2.18} & \multicolumn{1}{c|}{2.34}          & \multicolumn{1}{c|}{2.87} & \textbf{2.23} & \multicolumn{1}{c|}{1345} & \multicolumn{1}{c|}{1014} & \textbf{257} & \multicolumn{1}{c|}{\textbf{23}} & \multicolumn{1}{c|}{239} & 109   \\ \hline
			NYC Library              & 336                  & 54K                  & 210K                  & \multicolumn{1}{c|}{3.67} & \multicolumn{1}{c|}{1.72} & \multicolumn{1}{c|}{2.17}          & \multicolumn{1}{c|}{2.21} & \textbf{1.89} & \multicolumn{1}{c|}{294}  & \multicolumn{1}{c|}{222}  & \textbf{57}  & \multicolumn{1}{c|}{\textbf{4}}  & \multicolumn{1}{c|}{24}  & 21    \\ \hline
			Piazza del Popolo        & 336                  & 31K                  & 154K                  & \multicolumn{1}{c|}{4.63} & \multicolumn{1}{c|}{1.88} & \multicolumn{1}{c|}{2.54}          & \multicolumn{1}{c|}{2.33} & \textbf{2.20} & \multicolumn{1}{c|}{199}  & \multicolumn{1}{c|}{150}  & \textbf{38}  & \multicolumn{1}{c|}{\textbf{2}}  & \multicolumn{1}{c|}{14}  & 14    \\ \hline
			Piccadilly               & 2303                 & 185K                 & 797K                  & \multicolumn{1}{c|}{4.64} & \multicolumn{1}{c|}{2.11} & \multicolumn{1}{c|}{3.72}          & \multicolumn{1}{c|}{3.27} & \textbf{2.55} & \multicolumn{1}{c|}{972}  & \multicolumn{1}{c|}{733}  & \textbf{177} & \multicolumn{1}{c|}{\textbf{16}} & \multicolumn{1}{c|}{159} & 80    \\ \hline
			Roman Forum              & 1067                 & 227K                 & 1046K                 & \multicolumn{1}{c|}{4.20} & \multicolumn{1}{c|}{1.82} & \multicolumn{1}{c|}{2.14}          & \multicolumn{1}{c|}{2.79} & \textbf{1.90} & \multicolumn{1}{c|}{1400} & \multicolumn{1}{c|}{1056} & \textbf{279} & \multicolumn{1}{c|}{\textbf{21}} & \multicolumn{1}{c|}{221} & 108   \\ \hline
			Tower of London          & 484                  & 124K                 & 557K                  & \multicolumn{1}{c|}{5.14} & \multicolumn{1}{c|}{1.68} & \multicolumn{1}{c|}{4.48}          & \multicolumn{1}{c|}{2.61} & \textbf{2.50} & \multicolumn{1}{c|}{702}  & \multicolumn{1}{c|}{583}  & \textbf{127} & \multicolumn{1}{c|}{\textbf{12}} & \multicolumn{1}{c|}{101} & 57    \\ \hline
			Trafalgar                & 5067                 & 333K                 & 1286K                 & \multicolumn{1}{c|}{4.80} & \multicolumn{1}{c|}{2.11} & \multicolumn{1}{c|}{3.76}          & \multicolumn{1}{c|}{3.24} & \textbf{2.17} & \multicolumn{1}{c|}{1678} & \multicolumn{1}{c|}{1265} & \textbf{309} & \multicolumn{1}{c|}{\textbf{28}} & \multicolumn{1}{c|}{293} & 146   \\ \hline
			Union Square             & 816                  & 26K                  & 90K                   & \multicolumn{1}{c|}{6.77} & \multicolumn{1}{c|}{1.93} & \multicolumn{1}{c|}{3.71}          & \multicolumn{1}{c|}{3.32} & \textbf{2.91} & \multicolumn{1}{c|}{121}  & \multicolumn{1}{c|}{91}   & \textbf{22}  & \multicolumn{1}{c|}{\textbf{1}}  & \multicolumn{1}{c|}{8}   & 8     \\ \hline
			Vienna Cathedral         & 843                  & 157K                 & 504K                  & \multicolumn{1}{c|}{5.73} & \multicolumn{1}{c|}{1.88} & \multicolumn{1}{c|}{3.69}          & \multicolumn{1}{c|}{3.32} & \textbf{2.37} & \multicolumn{1}{c|}{723}  & \multicolumn{1}{c|}{545}  & \textbf{146} & \multicolumn{1}{c|}{\textbf{11}} & \multicolumn{1}{c|}{92}  & 55    \\ \hline
			Yorkminster              & 428                  & 101K                 & 377K                  & \multicolumn{1}{c|}{5.29} & \multicolumn{1}{c|}{2.02} & \multicolumn{1}{c|}{2.99}          & \multicolumn{1}{c|}{3.16} & \textbf{2.24} & \multicolumn{1}{c|}{542}  & \multicolumn{1}{c|}{409}  & \textbf{128} & \multicolumn{1}{c|}{\textbf{7}}  & \multicolumn{1}{c|}{59}  & 39    \\ \hline
		\end{tabular}%
	}
\end{table*}

\subsection{Evaluating Lazy Communication}
\label{sec:experiments:ablation}
%\myParagraph{Evaluating Lazy Communication}
We evaluate the proposed lazy communication scheme using the \dataset{Castle30} dataset \cite{Strecha08Benchmark}, which consists of 30 images observing a courtyard (\cref{fig:castle30:dataset}).
We use \software{Theia} \cite{TheiaSfM} to generate the input BA problem, which contains 23564 map points in total. 
We divide the BA problem into 30 agents and run \Algorithm for 50 iterations.
In this experiment, we find that it is sufficient to fix the preconditioner at the initial iteration, which corresponds to letting 
$\delta_p \to +\infty$ in \eqref{eq:lazy_jacobi_rule}.  
This is because for this relatively simple problem, the initial preconditioner already gives a good approximation of curvature information at all subsequent iterates.
Consequently, we mainly focus on evaluating parameters that affect the lazy communication of gradients \eqref{eq:lazy_reduced_grad_rule}.

We first evaluate the impact of $\epsilon_d$ in \eqref{eq:lazy_reduced_grad_rule}.
Intuitively, larger values of $\epsilon_d$ imply that agents are more tolerant of gradient approximation error, and hence communicate less at each iteration. 
We set all $\epsilon_d$ ($d = 1, \ldots, D$) to a common value $\epsilon$ and vary $\epsilon$ in our experiments.
To measure solution accuracy, we record the root-mean-square error (RMSE) of camera positions, computed after aligning with the ground truth via a similarity transformation.
\cref{fig:castle30:cam_error_vs_iter} shows the convergence of \Algorithm under varying values of $\epsilon$.
For comparison, we also include a reference solution computed by centralized optimization using \software{g2o}.
Except when using a very loose threshold of $\epsilon=100$ (red curve), 
lazy communication has minimal impact on the iterations of \Algorithm.
Furthermore, the communication efficiency of our method is clearly seen in \cref{fig:castle30:cam_error_vs_comm}, where we plot convergence as a function of total amount of uploads to the server.
To provide more insights, \cref{fig:castle30:lazy_comm_pattern}
visualizes the amount of gradient blocks uploaded to the server at each iteration. 
For each value of $\epsilon$, the corresponding solid line denotes the percentage of uploaded gradient blocks averaged across all agents, 
and the surrounding shaded area represents one standard deviation. 
Recall that choosing $\epsilon = 0$ forces all agents to upload all blocks at every iteration (blue curve in \cref{fig:castle30:lazy_comm_pattern}).
Our result clearly shows that varying $\epsilon$ provides an effective way to control the amount of uploads during optimization.

In addition, we also evaluate the impact of $\bar{d}$ on convergence.
Recall from \eqref{eq:lazy_reduced_grad_rule} that $\bar{d}$ determines the number of past gradients 
that are used to compute the triggering threshold. 
\cref{fig:castle30:vary_queue_size} shows the performance of \Algorithm under different choices of $\bar{d}$ with fixed $\epsilon = 5$. 
While the differences are not significant, our result still suggests that using more past gradients (\eg, $\bar{d} = 9$) helps to save more communication.

\subsection{Performance on Collaborative SLAM Datasets}
\label{sec:experiments:slam}
In this subsection, we evaluate \Algorithm on collaborative BA problems from multi-robot SLAM applications. 
We use the monocular version of \software{ORB-SLAM3}~\cite{ORBSLAM3} to extract BA problems from the \dataset{EuRoc} \cite{Burri16Euroc} and \dataset{KITTI} \cite{Geiger12KITTI} datasets.
Each \dataset{EuRoc} dataset contains multiple sequences recorded in the same indoor space,
and we use the multi-session feature of \software{ORB-SLAM3} to simulate each sequence as a single robot.
For each \dataset{KITTI} dataset, we divide the overall trajectory into multiple segments to simulate multiple robots.
We generate noisy inputs for each dataset by perturbing the \software{ORB-SLAM3} estimates by zero-mean Gaussian noise.\footnote{
	Specifically, the noise standard deviation for robot rotation, robot position, and map points are set to 5~deg, 0.1~m, 0.05~m for \textsf{EuRoc}, 
	and 5~deg, 2~m, 0.1~m for \textsf{KITTI}.
}

We compare \Algorithm against two baseline methods that can be implemented under the communication architecture considered in this work.
The first baseline is the method in \cite{Wu11MulticoreBA} using distributed preconditioned conjugate gradient (PCG). 
In our case, we use distributed PCG to solve the 
reduced second-order approximation in Lemma~\ref{lem:schur_complement}, 
where the problem is re-linearized after every 10 PCG iterations.
The second baseline is the Douglas-Rachford (DR) splitting method proposed in \cite{Eriksson16ConsensusBA}.
Similar to our method, both baseline methods only require agents to communicate information over the observed parts of the shared model (see \cref{remark:implementation}). 

\cref{tab:cslam} shows the performance of all algorithms after 50 iterations.
All results are averaged across 10 random runs.
We evaluate the RMSE absolute trajectory error (ATE) against ground truth, the mean reprojection error, as well as the total amount of uploads during optimization.
For comparison, we also include a reference solution computed by centralized optimization using \software{g2o}.
We note that the higher ATE in \dataset{KITTI} is due to the larger scale of the datasets.
As shown in \cref{tab:cslam}, \Algorithm achieves similar or better performance compared to baseline methods, while using significantly less communication.
Specifically, when compared to DR, \Algorithm achieves up to 65\% total communication reduction, clearly demonstrating the communication efficiency of our method.

\subsection{Performance on Collaborative SfM Datasets}
\label{sec:experiments:sfm}
We also evaluate \Algorithm on collaborative SfM scenarios using the \dataset{1DSfM} dataset \cite{Wilson141dsfm}, which contains 15 medium to large scale internet photo collections.
We use \software{Theia}~\cite{TheiaSfM} to generate the input BA problems.
Then, we partition each problem randomly to simulate a scenario with 50 agents.

Similar to the previous subsection, we evaluate the performance of all algorithms after 50 iterations.
\cref{tab:1dsfm} shows the results.
Since ground truth is not available, we only record the final mean reprojection error.
\Algorithm outperforms baseline methods in most datasets, and achieves final reprojection errors that are close to the centralized reference solutions.
Once again, \Algorithm demonstrates superior communication efficiency.
When compared to DR, our method achieves 68\%-78\% reduction in terms of total uploads.
Lastly, we also evaluate the average local iteration time of all methods (last three columns in \cref{tab:1dsfm}).
Our method is faster than DR, since the latter requires each agent to solve a smaller nonlinear optimization problem at every iteration.
On the other hand, the local iteration time of our method is larger than PCG.
However, considering the large size of the SfM datasets, an average iteration time ranging from $8$~ms to $293$~ms for our method is still reasonable, and can be improved by further optimizing our implementation (\eg, via additional parallelization).

%!TEX root = ../root.tex
\section{Conclusion}
We presented \Algorithm, a communication-efficient distributed algorithm for collaborative geometric estimation. 
Each iteration of \Algorithm allows agents to analytically eliminating private variables.
Furthermore, by incorporating lazy and partial aggregation at the server, \Algorithm achieves significant communication reduction, which makes it suitable for multi-robot and mixed reality applications subject to limited network bandwidth.
Under generic conditions, we proved that \Algorithm converges globally to first-order critical points with a sublinear convergence rate.
Evaluations on large-scale BA problems in collaborative SLAM and SfM scenarios 
show that \Algorithm performs competitively against existing techniques 
while achieving a consistent communication reduction of up to 78\%. 

\optional{
While our formulation is general, the current evaluation focuses on BA with the standard least squared costs.
In future work, we plan to evaluate on more applications and study the impact of robust cost functions.
In addition, it would be interesting to further improve the lazy communication scheme, \eg, by designing a single triggering condition that controls the communication of both preconditioners and gradients.
}

%%%%%%%%% REFERENCES
{\small
	\bibliographystyle{IEEEtran}
	\bibliography{ref/misc,ref/optimization,ref/sfm,ref/slam}
}

\clearpage
\onecolumn
\begin{appendices}
	%!TEX root = ../root.tex
\section{Proofs}
\label{sec:proof}
%!TEX root = ../root.tex

\subsection{Forming the Reduced Second-Order Approximation}
\label{sec:proof_formulation}
\begin{proof}[Proof of Lemma~\ref{lem:schur_complement}]
	Recall from \eqref{eq:global_quadratic_model} that the global second-order approximation $\mhat(u,v)$ involving both private vectors $u \in T_x \Xcal$ and shared vector $v \in T_y \Ycal$ is defined as,
	\begin{equation}
		\small
		\mhat(u,v) = 
		f(x,y) + 
		\bigg \langle
		{
			\bmat
			g_x \\ g_y
			\emat},
		\bmat
		u \\ v
		\emat
		\bigg \rangle
		+ \frac{1}{2}
		\bigg \langle
		\bmat
		u \\
		v
		\emat,
		{
			\bmat
			A        & C \\
			{C}^\top & B
			\emat}
		\bmat
		u \\
		v
		\emat
		\bigg \rangle.
		\label{pf:eq:global_quadratic_model}
	\end{equation}
	Setting the gradient of $\mhat(u,v)$ with respect to $u$ to zero yields,
	\begin{equation}
		\nabla_u \mhat(u,v) = g_x + Au + Cv = 0 \implies \ustar(v) = -A^{-1}(Cv + g_x).
		\label{pf:eq:ustar_def}
	\end{equation}
	Recall the definition of $A$ and $C$ in \eqref{eq:ABC_def}.
	In particular, since $A$ is a block-diagonal matrix, the $i$-th component of $\ustar(v)$
	(corresponding to agent $i$) is given by,
	\begin{equation}
		\ustar_i(v) = -A_i^{-1}(C_i v + g_{ix}), \;\; \forall i \in [N].
	\end{equation}
	Next, substitute $\ustar(v)$ defined in \eqref{pf:eq:ustar_def} into \eqref{pf:eq:global_quadratic_model}. After collecting terms, we obtain,
	\begin{align}
		\hhat(v) \triangleq \mhat(\ustar(v), v) = 
		f(x,y) - \frac{1}{2}\inner{g_x}{A^{-1}g_x} 
		+ \inner{
			\underbrace{g_y - C^\top A^{-1}g_x}_{w}
		}{v}
		+ \frac{1}{2} \inner{v}{
			\underbrace{\left( B - C^\top A^{-1} C \right)}_{S}
			v}.
		\label{pf:eq:reduced_model}
	\end{align}
	Consider the vector $w$ as defined in \eqref{pf:eq:reduced_model}.
	Note that the global Riemannian gradient with respect to $y$ 
	satisfies $g_y = \sum_{i=1}^N g_{iy}$ where $g_{iy} \triangleq \rgrad_y f_i(x_i, y)$.
	In addition, because of the block-diagonal structure of $A$ in \eqref{eq:ABC_def},
	$C^\top A^{-1}g_x = \sum_{i=1}^N C_i^\top A_i^{-1} g_{ix}$.
	Combining these results, we have that $w = \sum_{i=1}^N w_i$ where $w_i$ is defined as in \eqref{eq:w_def}.
	Similarly, for the matrix $S$ defined in \eqref{pf:eq:reduced_model}, it can be readily verified that $S = \sum_{i=1}^N S_i$ where $S_i$ is defined as in \eqref{eq:S_def}. 
\end{proof}

%!TEX root = ../root.tex

\subsection{Convergence Proofs for \cref{alg:lazy}}
\label{sec:proof_convergence}
We start by reviewing several notations that are needed in this section.
We use the superscript $k$ to denote the value of a variable at iteration $k$ of \cref{alg:lazy}.
For example, $g^k_x$ denotes the value of the Riemannian gradient $g_x$ 
introduced in \eqref{eq:global_quadratic_model} at iteration $k$.
Recall that $\norm{\cdot}$ (without subscript) denotes the standard norm 
associated with the Riemannian metric.
We also introduce an additional notation to simplify our presentation.
Recall that $P^k: T_{y^k} \Ycal \to T_{y^k} \Ycal$ is the preconditioner used at iteration $k$ to update the shared variable.
In the following, we use 
$\inner{\cdot}{\cdot}_k$ and 
$\norm{\cdot}_k$ as shortcuts for 
$\inner{\cdot}{\cdot}_{P^k}$ and
$\norm{\cdot}_{P^k}$, \ie,
\begin{align}
 \inner{v_1}{v_2}_k 	&\triangleq \inner{v_1}{v_2}_{P^k} = \inner{v_1}{P^k v_2}, \;\;
	\forall v_1, v_2 \in T_{y^k} \Ycal, \\
 \norm{v}_k 	&\triangleq \norm{v}_{P^k} = \sqrt{\inner{v}{P^k v}}, \;\; \forall v \in T_{y^k} \Ycal.
\end{align}

In order to establish the convergence of \cref{alg:lazy}, we start by analyzing the change in the global objective \eqref{eq:problem_formulation_objective} after a single iteration (one step change). 
The following lemma provides an upper bound on the change in objective value.
\begin{lemma}
	Under Assumption~\ref{as:assumptions},
	each iteration of \cref{alg:lazy} satisfies,
	\begin{equation}
		f(x^{k+1}, y^{k+1}) - f(x^k, y^k) \leq 
		- \frac{1}{2L} \norm{g^k_x}^2
		- \frac{\gamma}{2} \norm{w^k}^2_k
		- \left( \frac{\gamma}{2} - \frac{\sigmabar \gamma^2}{2} \right) \norm{\what^k}^2_k
		+ \frac{\gamma}{2} \norm{\sum_{i=1}^N (\what^k_i - w^k_i)}^2_k.
		\label{eq:pre_descent_lemma_lazy}
	\end{equation}
	\label{lem:pre_descent_lemma_lazy}
\end{lemma}
\begin{proof}
	Let $\mhat^k(\cdot)$ and $\hhat^k(\cdot)$ denote the second-order approximation in \eqref{eq:global_quadratic_model} and reduced second-order approximation in \eqref{eq:reduced_global_quadratic_model} at iteration $k$, respectively. 
	Substitute $v^k = -\gamma P^k \what^k$ into $\hhat^k(\cdot)$:
	\begin{align}
		\hhat^k(v^k) = f(x^k,y^k) 
		- \frac{1}{2} \inner{g^k_x}{(A^k)^{-1} g^k_x} 
		- \gamma \inner{w^k}{P^k\what^k} 
		+ \frac{\gamma^2}{2} \inner{P^k\what^k}{S^kP^k\what^k}.
		\label{eq:pre_descent_lemma_lazy_step_1}
	\end{align}
	In \eqref{eq:pre_descent_lemma_lazy_step_1}, recall that $w^k = \sum_{i=1}^N w^k_i$ is the true global gradient computed using latest local gradients from all agents. 
	On the other hand, $\what^k = \sum_{i=1}^N \what^k_i$ 
	is the approximate global gradient computed using the lazily uploaded gradients from agents.
	Rearrange the third term in the right hand side of \eqref{eq:pre_descent_lemma_lazy_step_1} as, 
	\begin{align}
		- \gamma \inner{w^k}{P^k\what^k} &= - \gamma \inner{w^k}{P^k(w^k + \what^k - w^k)} \\
		&= - \gamma \norm{w^k}^2_k - \gamma \inner{w^k}{P^k(\what^k - w^k)} \\
		&= - \gamma \norm{w^k}^2_k - \gamma \inner{w^k}{P^k\sum_{i=1}^N (\what^k_i - w^k_i)} \\
		&= - \gamma \norm{w^k}^2_k + \inner{-\sqrt{\gamma P^k}  w^k}{\sqrt{\gamma P^k} \sum_{i=1}^N (\what^k_i - w^k_i)}.
		\label{eq:pre_descent_lemma_lazy_step_2}
	\end{align}
	In \eqref{eq:pre_descent_lemma_lazy_step_2}, the matrix $\gamma P^k$ is positive definite,
	and we use $\sqrt{\gamma P^k}$ to denote its matrix square root.
	Recall the equality $\inner{a}{b} = \frac{1}{2}\twonorm{a}^2 + \frac{1}{2} \twonorm{b}^2 - \frac{1}{2} \twonorm{a-b}^2$. 
	Applying this equality on the inner product term in \eqref{eq:pre_descent_lemma_lazy_step_2}, we obtain,
	\begin{align}
		\inner{-\sqrt{\gamma P^k}  w^k}{\sqrt{\gamma P^k} \sum_{i=1}^N (\what^k_i - w^k_i)}
		= \frac{\gamma}{2} \norm{w^k}^2_k
		+ \frac{\gamma}{2} \norm{\sum_{i=1}^N (\what^k_i - w^k_i)}^2_k
		- \frac{\gamma}{2} \norm{\what^k}^2_k.
		\label{eq:pre_descent_lemma_lazy_step_3}
	\end{align}
	Substitute \eqref{eq:pre_descent_lemma_lazy_step_3} into \eqref{eq:pre_descent_lemma_lazy_step_2}, we obtain, 
	\begin{align}
		- \gamma \inner{w^k}{P^k\what^k} &= 
		- \frac{\gamma}{2} \norm{w^k}^2_k 
		- \frac{\gamma}{2} \norm{\what^k}^2_k 
		+ \frac{\gamma}{2} \norm{\sum_{i=1}^N (\what^k_i - w^k_i)}^2_k.
		\label{eq:pre_descent_lemma_lazy_step_4}
	\end{align}
	Now, let us focus on the last term in \eqref{eq:pre_descent_lemma_lazy_step_1}. 
	Applying the Cauchy-Schwartz inequality with respect to the norm induced by $P^k$, it holds that, 
	\begin{align}
		\frac{\gamma^2}{2} \inner{P^k\what^k}{S^kP^k\what^k} 
		&= \frac{\gamma^2}{2} \inner{\what^k}{S^kP^k\what^k}_k \\
		&\leq \frac{\gamma^2}{2} \norm{\what^k}_k \norm{S^kP^k\what^k}_k \\
		&\leq \frac{\sigmabar \gamma^2}{2} \norm{\what^k}^2_k.
		\label{eq:pre_descent_lemma_lazy_step_5}
	\end{align}
	The last inequality \eqref{eq:pre_descent_lemma_lazy_step_5} holds because of assumption \eqref{as:preconditioner}.
	Finally, substitute \eqref{eq:pre_descent_lemma_lazy_step_4} and \eqref{eq:pre_descent_lemma_lazy_step_5} into \eqref{eq:pre_descent_lemma_lazy_step_1}, we obtain that, 
	\begin{align}
		\hhat^k(v^k) &\leq 
		f(x^k,y^k) 
		- \frac{1}{2} \inner{g^k_x}{(A^k)^{-1} g^k_x} 
		- \frac{\gamma}{2} \norm{w^k}^2_k
		- \left( \frac{\gamma}{2} - \frac{\sigmabar \gamma^2}{2} \right) \norm{\what^k}^2_k
		+ \frac{\gamma}{2} \norm{\sum_{i=1}^N (\what^k_i - w^k_i)}^2_k \\
		&\leq 
		f(x^k,y^k) 
		- \frac{1}{2L} \norm{g^k_x}^2
		- \frac{\gamma}{2} \norm{w^k}^2_k
		- \left( \frac{\gamma}{2} - \frac{\sigmabar \gamma^2}{2} \right) \norm{\what^k}^2_k
		+ \frac{\gamma}{2} \norm{\sum_{i=1}^N (\what^k_i - w^k_i)}^2_k,
	\end{align}
	where the last inequality holds, because $A^k \preceq L I$ due to assumption~\eqref{as:hessian_apx_bound}.
	To conclude the proof, note that \eqref{as:lipschitz_type_gradient} and \eqref{as:hessian_apx_bound} together imply that the model function $\mhat^k$ is an upper bound on the current pullback function $\fhat^k$,
	\begin{align}
		\fhat^k(u,v) 
		&\leq f(x^k, y^k) 
		+ \left \langle
		\bmat
		g^k_x \\ g^k_y
		\emat,
		\bmat
		u \\ v
		\emat
		\right \rangle
		+ 
		\frac{c_g}{2}
		\norm{\bmat
		u \\ v
		\emat}^2 \\
		& \leq f(x^k, y^k) 
		+ \left \langle
		\bmat
		g^k_x \\ g^k_y
		\emat,
		\bmat
		u \\ v
		\emat
		\right \rangle
		+ 
		\frac{1}{2}
		\bigg \langle
		\bmat
		u \\
		v
		\emat,
		M^k
		\bmat
		u \\
		v
		\emat
		\bigg \rangle \\
		& \triangleq
		\mhat^k(u,v).
	\end{align}
	Above, the first inequality holds due to the Lipschitz-type gradient conditions for pullback \eqref{as:lipschitz_type_gradient},
	and the second inequality holds because $M^k \succeq c_g I$ \eqref{as:hessian_apx_bound}.
	The above inequality directly shows that,
	\begin{align}
		f(x^{k+1}, y^{k+1}) = \fhat^k(u^k, v^k) 
		\leq \mhat^k(u^k, v^k) 
		= \mhat^k(\ustar(v^k), v^k) 
		= \hhat^k(v^k).
	\end{align}
	This concludes the proof.
\end{proof}

In Lemma~\ref{lem:pre_descent_lemma_lazy}, the RHS of \eqref{eq:pre_descent_lemma_lazy}
bounds the absolute reduction in the global cost function after each iteration of \cref{alg:lazy}.
However, due to the last term in \eqref{eq:pre_descent_lemma_lazy} (which captures the error between the approximate gradient $\what^k$ and true $w^k$), the RHS can in general be positive. 
This means that we cannot directly use \eqref{eq:pre_descent_lemma_lazy} to show that 
\cref{alg:lazy} decreases the global cost function at every iteration, \ie, \cref{alg:lazy} is not a descent method with respect to the cost function $f$.
Fortunately, we can still show that \cref{alg:lazy} is a descent method
with respect to a \emph{Lyapunov function}, which is sufficient for proving convergence.
This proof technique is inspired by LAG~\cite{Chen18LAG}.
Specifically, we define the Lyapunov function to be the sum of global objective and the squared norms of past approximate gradients,
\begin{equation}
V^k \triangleq f(x^k, y^k) + \sum_{d=1}^{\bar{d}} \beta_d \norm{\what^{k-d}}^2_{k-d},
\label{eq:lyapunov}
\end{equation}
where $\{\beta_d \geq 0, d = 1, \hdots, \bar{d}\}$ are constants to be specified.
Note that $V^k$ combines the current cost function with weighted squared norms of past approximate reduced gradients.
Intuitively, these squared norms account for the approximation errors induced by lazy communication, and allows us to establish the convergence of the proposed method.

\begin{lemma}[Descent lemma]
	Under Assumption~\ref{as:assumptions}, there exist suitable choices of parameters $\gamma, \{\beta_d\}, \{\epsilon_d\}$, such that
	each iteration of Algorithm~\ref{alg:lazy} satisfies,
	\begin{equation}
	V^{k+1} - V^k \leq
	- \frac{1}{2L} \norm{g^k_x}^2
	- \alpha_0 \norm{w^k}^2_k 
	- \sum_{d=1}^{\bar{d}} \alpha_d \norm{\what^{k-d}}^2_{k-d}
	\leq 0,
	\label{eq:descent_lemma_lazy}
	\end{equation}
	where $\alpha_0, \hdots, \alpha_{\bar{d}} > 0$ are fixed constants.
	In particular, the following provides a set of admissible conditions on the parameters such that \eqref{eq:descent_lemma_lazy} holds,
	\begin{align}
	    0       &< \gamma < 1/\sigma_p, \label{eq:descent_lemma_lazy:stepsize} \\
	    \beta_1 &= (\gamma - \sigmabar \gamma^2)/2, \\
	    \beta_d &< \beta_{d-1} - \gamma \epsilon_{d-1}/2, \; d = 2, \hdots, \bar{d}, \\
	    \beta_{\bar{d}} &> \gamma \epsilon_{\bar{d}} /2. \label{eq:descent_lemma_lazy:beta_dbar}
	\end{align}
	\label{lem:descent_lemma_lazy}
\end{lemma}
\begin{proof}
	Consider the difference between the Lyapunov function between the current and next iterations,
	\begin{align}
	V^{k+1} - V^k = f(x^{k+1}, y^{k+1}) - f(x^k, y^k) + 
	\sum_{d=1}^{\bar{d}} \beta_d \norm{\what^{k-d+1}}^2_{k-d+1} - \sum_{d=1}^{\bar{d}} \beta_d \norm{\what^{k-d}}^2_{k-d}.
	\end{align}
	Using Lemma~\ref{lem:pre_descent_lemma_lazy}, we can obtain the following upper bound,
	\begin{equation}
	\begin{aligned}
	V^{k+1} - V^k \leq
	&- \frac{1}{2L} \norm{g^k_x}^2
	- \frac{\gamma}{2} \norm{w^k}^2_k
	- \left( \frac{\gamma}{2} - \frac{\sigmabar \gamma^2}{2} \right) \norm{\what^k}^2_k
	+ \frac{\gamma}{2} \norm{\sum_{i=1}^N (\what^k_i - w^k_i)}^2_k \\
	&+ \sum_{d=1}^{\bar{d}} \beta_d \norm{\what^{k-d+1}}^2_{k-d+1} - \sum_{d=1}^{\bar{d}} \beta_d \norm{\what^{k-d}}^2_{k-d}.
	\end{aligned}
	\end{equation}
	Grouping terms that involve $\norm{\what^k}^2_k$ together, we obtain,
	\begin{equation}
	\begin{aligned}
	V^{k+1} - V^k \leq
	&- \frac{1}{2L} \norm{g^k_x}^2
	- \frac{\gamma}{2} \norm{w^k}^2_k
	+ \left( \beta_1 - \frac{\gamma}{2} + \frac{\sigmabar \gamma^2}{2} \right) \norm{\what^k}^2_k 
	+ \frac{\gamma}{2} \norm{\sum_{i=1}^N (\what^k_i - w^k_i)}^2_k \\
	&+ \sum_{d=2}^{\bar{d}} \beta_d \norm{\what^{k-d+1}}^2_{k-d+1} - \sum_{d=1}^{\bar{d}} \beta_d \norm{\what^{k-d}}^2_{k-d}.
	\end{aligned}
	\label{eq:descent_lemma_step_1}
	\end{equation}
	Next, we obtain an upper bound for $\norm{\what^k}^2_k$. First, using the definition of the approximate gradient $\what^k$, it holds that,
	\begin{align}
	\norm{\what^k}^2_k = \norm{w^k + \sum_{i=1}^N (\what^k_i - w^k_i) }^2_k 
	= \norm{\sqrt{P^k} w^k + \sqrt{P^k} \sum_{i=1}^N (\what^k_i - w^k_i)}^2.
	\end{align}
	Next, applying Young's inequality, we arrive at the following upper bound, 
	\begin{align}
	\norm{\what^k}^2_k 
	&\leq (1+\rho)\norm{\sqrt{P^k} w^k}^2 + (1+\rho^{-1})\norm{\sqrt{P^k} \sum_{i=1}^N (\what^k_i - w^k_i)}^2 \\
	&= (1+\rho)\norm{w^k}^2_k + (1+\rho^{-1})\norm{\sum_{i=1}^N (\what^k_i - w^k_i)}^2_k,
	\label{eq:descent_lemma_step_2}
	\end{align}
	where $\rho > 0$ is any constant.
	Plug \eqref{eq:descent_lemma_step_2} into \eqref{eq:descent_lemma_step_1}. After grouping terms, we arrive at,
	\begin{equation}
	\begin{aligned}
	V^{k+1} - V^k \leq
	&- \frac{1}{2L} \norm{g^k_x}^2
	+ \left[\left(\beta_1 - \frac{\gamma}{2} + \frac{\sigmabar \gamma^2}{2}\right) (1+\rho) -\frac{\gamma}{2} \right] \norm{w^k}^2_k \\
	&+ \left[\left(\beta_1 - \frac{\gamma}{2} + \frac{\sigmabar \gamma^2}{2}\right) (1+\rho^{-1}) 
	+ \frac{\gamma}{2} \right] \norm{\sum_{i=1}^N (\what^k_i - w^k_i)}^2_k \\
	&+ \sum_{d=1}^{\bar{d}-1} (\beta_{d+1} - \beta_d) \norm{\what^{k-d}}^2_{k-d} - \beta_{\bar{d}} \norm{\what^{k-\bar{d}}}^2_{k-\bar{d}}.
	\end{aligned}
	\label{eq:descent_lemma_step_3}
	\end{equation}
	Next, we obtain an upper bound for the second row of \eqref{eq:descent_lemma_step_3}. Note that, 
	\begin{align}
	\norm{\sum_{i=1}^N (\what^k_i - w^k_i)}^2_k 
	&= \norm{\sum_{i=1}^N \sqrt{P^k} (\what^k_i - w^k_i)}^2 \\
	&\leq N \sum_{i=1}^N \norm{\sqrt{P^k} (\what^k_i - w^k_i)}^2 \\
	&=    N \sum_{i=1}^N \norm{\what^k_i - w^k_i}^2_k.
	\label{eq:descent_lemma_step_4}
	\end{align}
	Recall the communication triggering condition \eqref{eq:lazy_reduced_grad_rule}. By definition, \eqref{eq:lazy_reduced_grad_rule} guarantees that the approximation error for each block $l$ is upper bounded as follows,
	\begin{equation}
		\norm{\what^k_{il} - w^k_{il}}^2_{P^k_l} \leq \frac{1}{m N^2} \sum_{d=1}^{\bar{d}} \epsilon_d \norm{\what^{k-d}}^2_{k-d}.
	\end{equation}
	Summing the above inequality over all $m$ blocks, we can obtain an upper bound on the approximation error for the entire local gradient,
	\begin{equation}
		\norm{\what^k_{i} - w^k_{i}}^2_{k} = 
		\sum_{l=1}^{m} \norm{\what^k_{il} - w^k_{il}}^2_{P^k_l} 
		\leq \frac{1}{N^2} \sum_{d=1}^{\bar{d}} \epsilon_d \norm{\what^{k-d}}^2_{k-d}.
		\label{eq:descent_lemma_local_upper_bound}
	\end{equation}
	Substitute \eqref{eq:descent_lemma_local_upper_bound} into \eqref{eq:descent_lemma_step_4},
	\begin{align}
	\norm{\sum_{i=1}^N (\what^k_i - w^k_i)}^2_k
	&\leq N \sum_{i=1}^N \left( \frac{1}{N^2} \sum_{d=1}^{\bar{d}} \epsilon_d \norm{\what^{k-d}}^2_{k-d} \right) \\
	&=\sum_{d=1}^{\bar{d}} \epsilon_d \norm{\what^{k-d}}^2_{k-d}.
	\label{eq:descent_lemma_step_5}
	\end{align}
	Substitute \eqref{eq:descent_lemma_step_5} into the second row of \eqref{eq:descent_lemma_step_3},
	\begin{equation}
	\begin{aligned}
	V^{k+1} - V^k \leq
	&- \frac{1}{2L} \norm{g^k_x}^2
	+ \left[\left(\beta_1 - \frac{\gamma}{2} + \frac{\sigmabar \gamma^2}{2}\right) (1+\rho) -\frac{\gamma}{2} \right] \norm{w^k}^2_k \\
	&+ \left[\left(\beta_1 - \frac{\gamma}{2} + \frac{\sigmabar \gamma^2}{2} \right) (1+\rho^{-1}) 
	+ \frac{\gamma}{2} \right] \sum_{d=1}^{\bar{d}} \epsilon_d \norm{\what^{k-d}}^2_{k-d} \\
	&+ \sum_{d=1}^{\bar{d}-1} (\beta_{d+1} - \beta_d) \norm{\what^{k-d}}^2_{k-d} - \beta_{\bar{d}} \norm{\what^{k-\bar{d}}}^2_{k-\bar{d}}.
	\end{aligned}
	\label{eq:descent_lemma_step_6}
	\end{equation}
	After grouping terms in \eqref{eq:descent_lemma_step_6}, we arrive at,
	\begin{equation}
	\begin{aligned}
	V^{k+1} - V^k \leq
	&- \frac{1}{2L} \norm{g^k_x}^2
	- \left[ \frac{\gamma}{2} - \left( \beta_1 - \frac{\gamma}{2} + \frac{\sigmabar \gamma^2}{2}  \right) (1+\rho) \right] \norm{w^k}^2_{k} \\
	&- 
	\sum_{d=1}^{\bar{d}-1} \left \{ 
	\beta_d -
	\beta_{d+1} - 
	\epsilon_d \left[\left(\beta_1 - \frac{\gamma}{2} + \frac{\sigmabar \gamma^2}{2}  \right) (1+\rho^{-1}) 
	+ \frac{\gamma}{2} \right]  \right \} \norm{\what^{k-d}}^2_{k-d} \\
	&-
	\left \{ \beta_{\bar{d}}-
	\epsilon_{\bar{d}} \left[\left( \beta_1 - \frac{\gamma}{2} + \frac{\sigmabar \gamma^2}{2} \right) (1+\rho^{-1}) 
	+ \frac{\gamma}{2} \right]  \right \} \norm{\what^{k-\bar{d}}}^2_{k-\bar{d}}.
	\end{aligned}
	\end{equation}
	Define the following constants that correspond to the coefficients in the above inequality,
	\begin{align}
	\alpha_0 &\triangleq \frac{\gamma}{2} - \left( \beta_1 - \frac{\gamma}{2} + \frac{\sigmabar \gamma^2}{2}  \right) (1+\rho), \\
	\alpha_d &\triangleq \beta_d -
	\beta_{d+1} - 
	\epsilon_d \left[\left( \beta_1 - \frac{\gamma}{2} + \frac{\sigmabar \gamma^2}{2}  \right) (1+\rho^{-1}) 
	+ \frac{\gamma}{2} \right], \; d = 1, \hdots, \bar{d}-1, \\
	\alpha_{\bar{d}} &\triangleq \beta_{\bar{d}}-
	\epsilon_{\bar{d}} \left[\left( \beta_1 - \frac{\gamma}{2} + \frac{\sigmabar \gamma^2}{2}  \right) (1+\rho^{-1}) 
	+ \frac{\gamma}{2} \right].
	\end{align}
	For the Lyapunov function to be decreasing, it suffices to choose $\gamma, \{\epsilon_d\}, \{\beta_d\}$ such that $\alpha_d > 0$ for all $d = 0, 1, \hdots, \bar{d}$. 
	To conclude the proof, we show that the conditions outlined in \eqref{eq:descent_lemma_lazy:stepsize}-\eqref{eq:descent_lemma_lazy:beta_dbar}, which are inspired by similar conditions in \cite{Chen18LAG}, indeed 
	satisfy this requirement.
	Let us assume that $\{\beta_d\}$ is a decreasing sequence, \ie,
	$\beta_1 > \beta_2 > \hdots > \beta_{\bar{d}}$.
	This assumption makes intuitive sense, since it assigns larger weights to more recent gradients in the definition of the Lyapunov function \eqref{eq:lyapunov}.
	In addition, let us also assume that the stepsize satisfies $0 < \gamma < 1/\sigmabar$.
	Note the similarity of this assumption with the condition $0 < \gamma < 1/L$ that is commonly used to ensure the convergence of gradient descent (\eg, see \cite{Boumal2018Convergence}).
	Under these two simplifications, 
	let us choose $\beta_1 = (\gamma - \sigmabar \gamma^2)/2 > 0$, so that 
	$\beta_1 - \frac{\gamma}{2} + \frac{\sigmabar \gamma^2}{2} = 0$. 
	Then, it can be verified that the following conditions ensure
	$\alpha_d > 0$ for all $d = 0, 1, \hdots, \bar{d}$:
	\begin{align}
    	& \beta_d - \beta_{d+1} -\epsilon_d \gamma/2 > 0, \; d = 1, \hdots, \bar{d}-1,\\
    	& \beta_{\bar{d}} - \epsilon_{\bar{d}} \gamma/2 > 0.
	\end{align}
	We can verify that the above conditions are equivalent to \eqref{eq:descent_lemma_lazy:stepsize}-\eqref{eq:descent_lemma_lazy:beta_dbar}.
\end{proof}

\begin{remark}[Intuitions behind parameter settings]
\normalfont
Before proceeding, let us provide more insights on the choice of algorithm parameters \eqref{eq:descent_lemma_lazy:stepsize}-\eqref{eq:descent_lemma_lazy:beta_dbar} outlined in Lemma~\ref{lem:descent_lemma_lazy}.
For this purpose, let us focus on the special case when $\bar{d} = 1$, \ie, only a single past gradient is used in the calculation of the communication triggering condition \eqref{eq:lazy_reduced_grad_rule}.
In this case, it can be shown that the conditions \eqref{eq:descent_lemma_lazy:stepsize}-\eqref{eq:descent_lemma_lazy:beta_dbar} reduce to the following,
\begin{align}
    0       &< \gamma < 1/\sigma_p, \\
    \beta_1 &= (\gamma - \sigmabar \gamma^2)/2, \\
    \epsilon_1 &< 2\beta_1/\gamma = 1 - \sigma_p \gamma.
\end{align}
In particular, the last inequality demonstrates the intuitive \emph{trade-off} between the stepsize $\gamma$ and the lazy communication threshold $\epsilon_1$: 
\emph{with smaller stepsize, we can tolerate larger approximation errors (and hence save more communication) at each iteration}.
\end{remark}

With Lemma~\ref{lem:descent_lemma_lazy}, we are ready to prove Theorem~\ref{thm:convergence} which is stated in \cref{sec:convergence} and is repeated below.

\myParagraph{Theorem~\ref{thm:convergence}}
Under Assumption~\ref{as:assumptions} and the conditions in Lemma~\ref{lem:descent_lemma_lazy}, 
after $K$ iterations, the iterates generated by \cref{alg:lazy} satisfy,
\begin{equation}
\min_{k \in [K]} \norm{\rgrad f(x^k, y^k)}^2 = O(1/K). 
\end{equation}

\begin{proof}
	Using Lemma~\ref{lem:descent_lemma_lazy} we know that, 
	\begin{equation}
		\frac{1}{2L}\norm{g^k_x}^2
		+ \alpha_0 \norm{w^k}^2_k 
		\leq 	V^k - V^{k+1} 
	\end{equation}
	Furthermore, from Assumption~\eqref{as:preconditioner}, 
	\begin{equation}
		\frac{1}{2L}\norm{g^k_x}^2
		+ \alpha_0 \mu_p \norm{w^k}^2 
		\leq 	V^k - V^{k+1} 
	\end{equation}
	Define $\alpha \triangleq \min(1/2L, \alpha_0 \mu_p)$,
	\begin{equation}
	\alpha 
	\left (\norm{g^k_x}^2
	+ \norm{w^k}^2 \right )
	\leq 	V^k - V^{k+1} 
	\label{eq:convergence_step_1}
	\end{equation}
	A telescoping sum of \eqref{eq:convergence_step_1} from $k=1$ to $k=K$ yields,
	\begin{equation}
		\alpha \sum_{k=1}^K 
		\left( \norm{g^k_x}^2
		+ \norm{w^k}^2 \right) 
		\leq 
		V^1 - V^{K+1}
		\leq V^1 - f^\star.
		\label{eq:convergence_step_2}
	\end{equation}
	Above, $f^\star$ denotes the global minimum of Problem~\ref{problem_formulation}.
	The second inequality holds, because by definition of the Lyapunov function \eqref{eq:lyapunov} we have 
	$V^k \geq f^\star$ for all $k$.
	Inequality \eqref{eq:convergence_step_2} further implies that,
	\begin{equation}
		\min_{k \in [K]} \left( \norm{g^k_x}^2 + \norm{w^k}^2 \right) \leq \frac{V^1 - f^\star}{\alpha K}.
		\label{eq:convergence_step_3}
	\end{equation}
	To conclude the proof, we show that \eqref{eq:convergence_step_3} implies \eqref{eq:convergence} in Theorem~\ref{thm:convergence}. 
	From now on, let $k$ denote the iteration that minimizes the LHS of \eqref{eq:convergence_step_3}.
	In addition, define $\varepsilon \triangleq \sqrt{(V^1 - f^\star)/(\alpha K)}$.
	Then, \eqref{eq:convergence_step_3} implies that we have both $\norm{g_x^k} \leq \varepsilon$
	and $\norm{w^k} \leq \varepsilon$. 
	Recall from \eqref{pf:eq:reduced_model} that $w = g_y - C^\top A^{-1} g_x$.
	From \eqref{as:hessian_apx_bound} we have that $\norm{A^{-1}} \leq \mu^{-1}$.
	Also, since the approximate Hessian $M$ \eqref{eq:global_quadratic_model} is positive definite,
	\begin{equation}
		M = \bmat
		A      & C \\
		C^\top & B
		\emat \succeq 0,
	\end{equation}
	it holds that $C = A^{1/2} Z B^{1/2}$ where $\norm{Z} \leq 1$ \cite[Lemma~3.5.12]{Horn91Book}. 
	Therefore we also have that $\norm{C} \leq L$.
	Applying these results together with the triangle inequality,
	\begin{equation}
		\norm{g^k_y} \leq \norm{w^k} + \norm{C^\top A^{-1} g^k_x} 
		\leq \left( 1 + \frac{L}{\mu} \right) \varepsilon. 
	\end{equation}
	Lastly, for $g^k \triangleq \rgrad f(x^k, y^k)$,
	it holds that 
	\begin{equation}
		\norm{g^k}^2 = \norm{g^k_x}^2 + \norm{g^k_y}^2 
		\leq \left [  1 + \left( 1 + \frac{L}{\mu} \right)^2 \right ] \varepsilon^2
		=  \left [  1 + \left( 1 + \frac{L}{\mu} \right)^2 \right ] \frac{V^1 - f^\star}{\alpha K}.
	\end{equation}
	The proof is completed.
\end{proof}

	%!TEX root = ../root.tex
\section{Additional Tables and Figures}
\label{sec:additional}
\begin{table}[h]
	\setlength{\tabcolsep}{1.5pt}
	\renewcommand{\arraystretch}{1.5}
	\centering
	\caption{\small Default parameters of \Algorithm used in our experiments (\cref{sec:experiments}).}
	\begin{tabular}{|c|c|l|}
		\hline
		Parameter    & Value  & \multicolumn{1}{c|}{Description}                           \\ \hline
		$\gamma$     & 1.0    & Stepsize used to update shared variable \eqref{eq:server_update}.                 \\ \hline
		$\lambda$    & $10^6$ & Regularization parameter in LM.               \\ \hline
		$\epsilon_d$ & 10     & Parameter used to compute the lazy communication threshold \eqref{eq:lazy_reduced_grad_rule}. \\ \hline
		$\bar{d}$ & 10 & Number of past gradients considered when computing lazy communication threshold \eqref{eq:lazy_reduced_grad_rule}. \\ \hline
		$\delta_p$   & 0.1    & Parameter used for updating the lazy Jacobi preconditioner \eqref{eq:lazy_jacobi_rule}. \\ \hline
	\end{tabular}
	\label{tab:default_parameters}
	%\vspace{-5cm}
\end{table}

\begin{figure}[h]
	\centering
	\begin{subfigure}[t]{0.45\linewidth}
		\includegraphics[width=\textwidth, trim=0 200 0 250, clip]{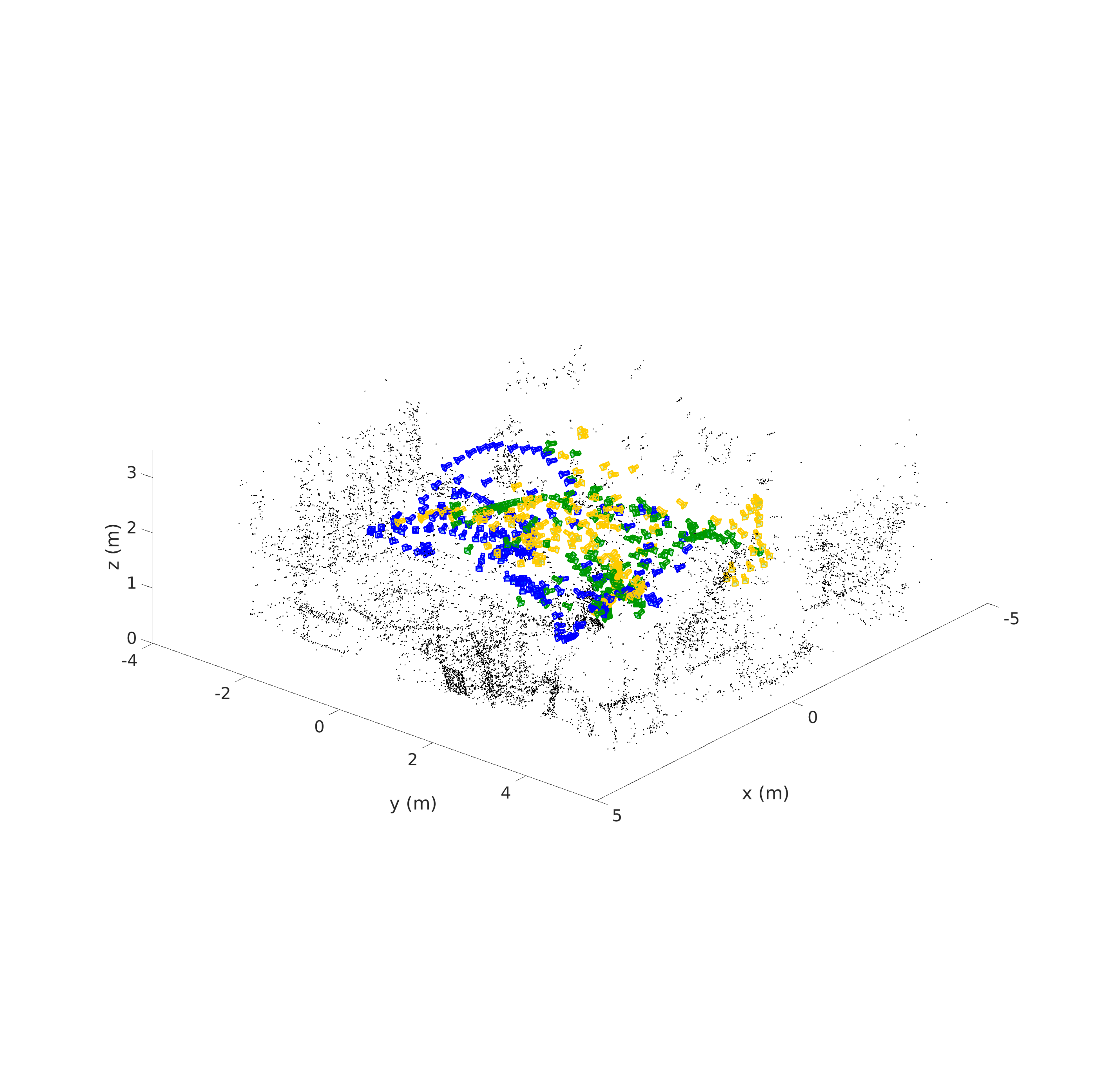}
		\caption{\dataset{Vicon Room 1}}
		\label{fig:cslam:vicon_room_1}
	\end{subfigure}
	~
	\begin{subfigure}[t]{0.45\linewidth}
		\includegraphics[width=\textwidth, trim=0 200 0 250, clip]{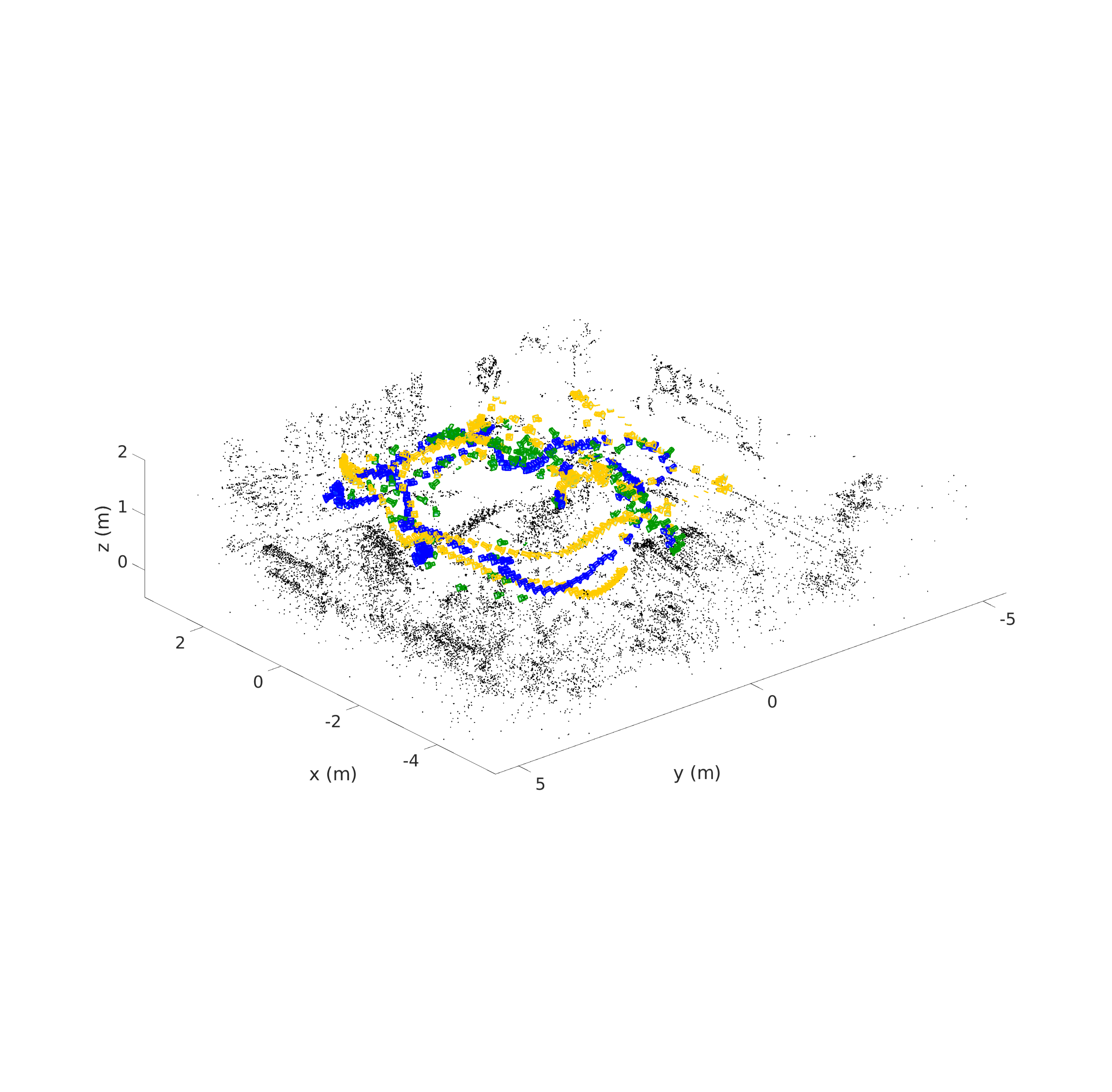}
		\caption{\dataset{Vicon Room 2}}
		\label{fig:cslam:vicon_room_2}
	\end{subfigure}	
	\\
	\begin{subfigure}[t]{0.45\linewidth}
		\includegraphics[width=\textwidth, trim=0 100 0 100, clip]{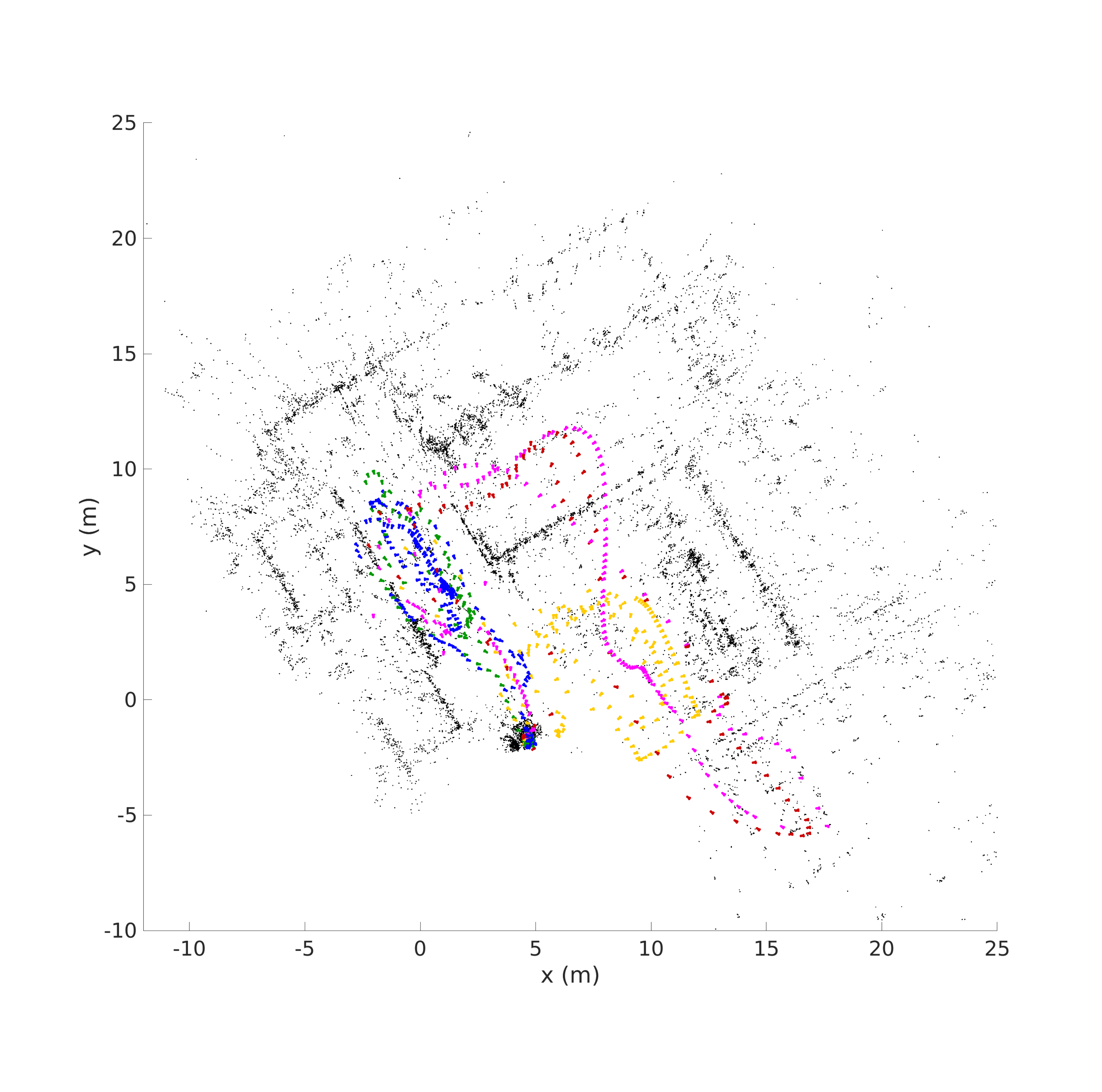}
		\caption{\dataset{Machine Hall}}
		\label{fig:cslam:machine_hall}
	\end{subfigure}
	~
	\begin{subfigure}[t]{0.45\linewidth}
		\includegraphics[width=\textwidth, trim=0 100 0 100, clip]{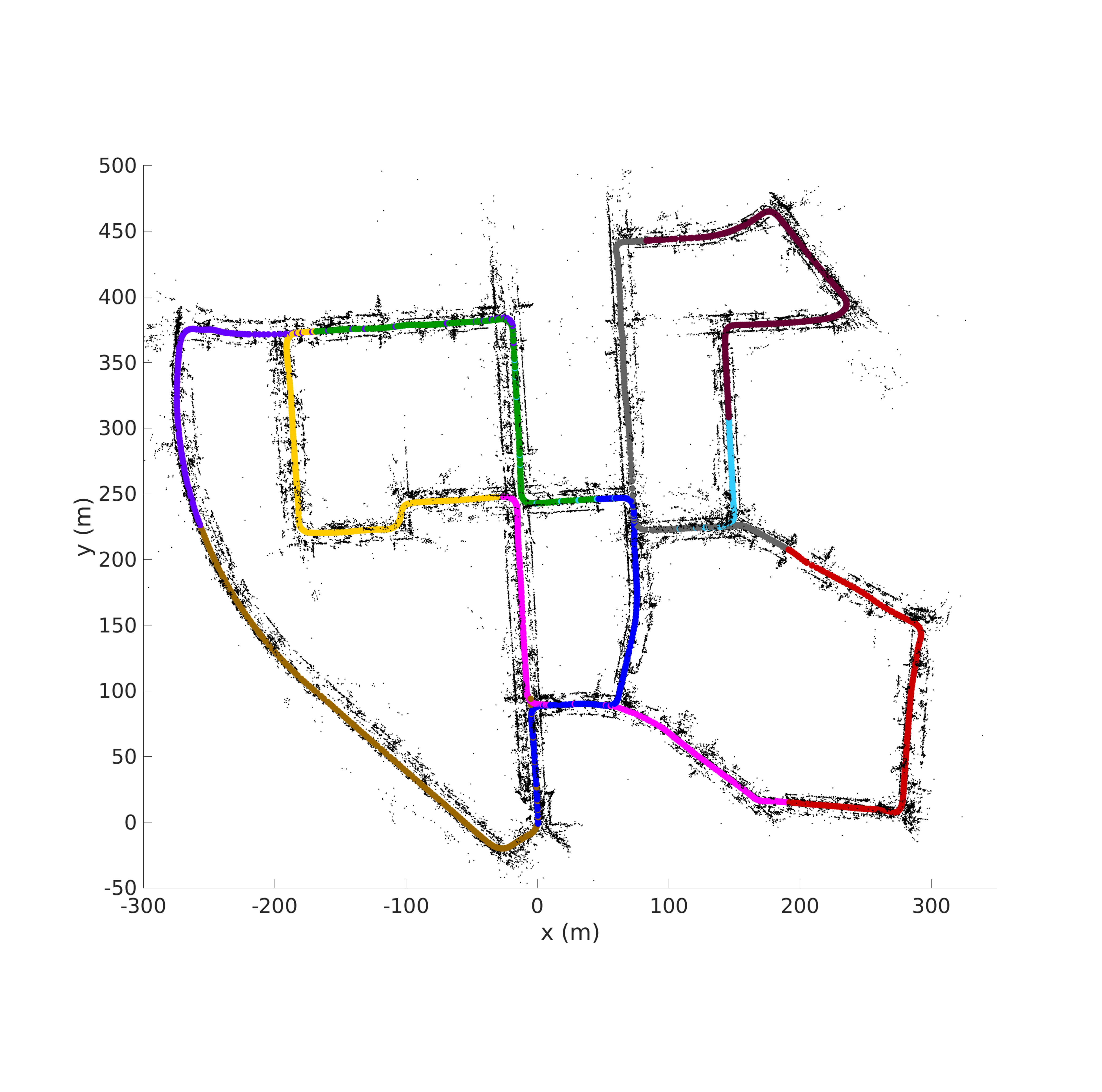}
		\caption{\dataset{KITTI~00}}
		\label{fig:cslam:kitti_00}
	\end{subfigure}
	\\
	\begin{subfigure}[t]{0.60\linewidth}
		\includegraphics[width=\textwidth, trim=0 600 0 300, clip]{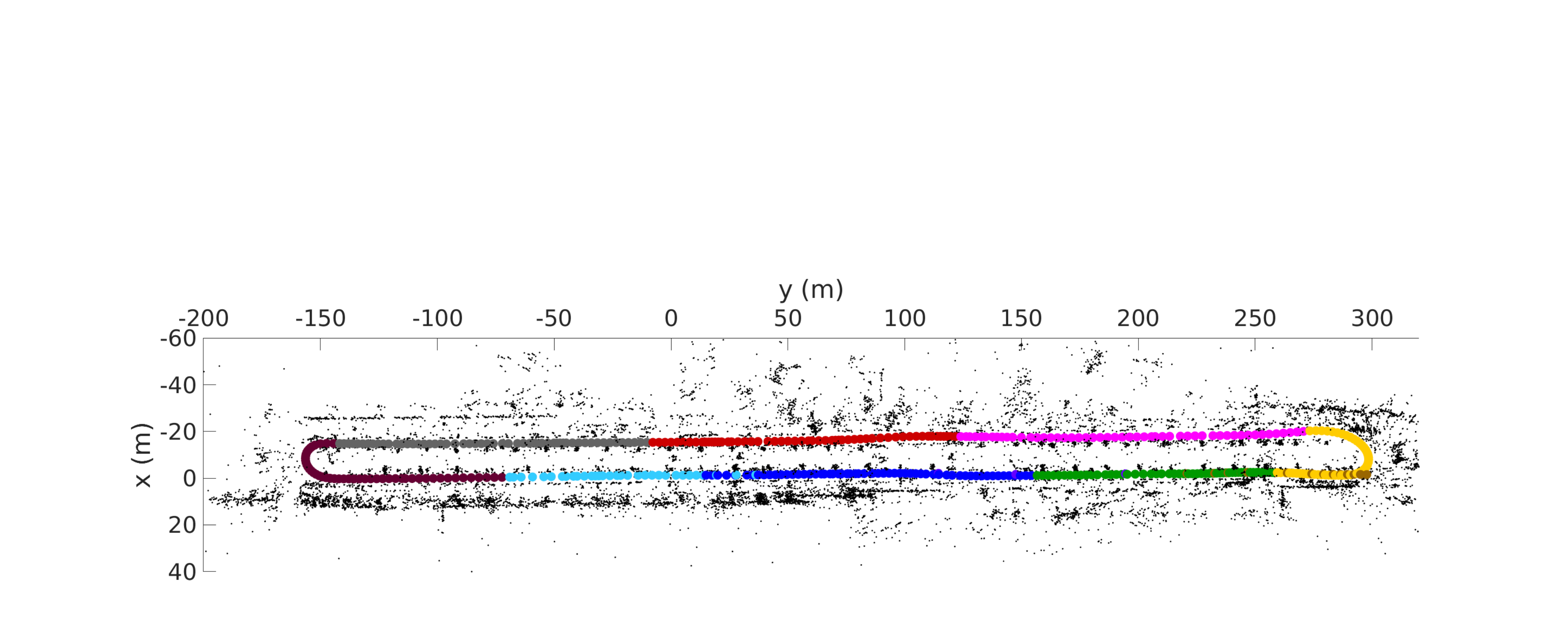}
		\caption{\dataset{KITTI~06}}
		\label{fig:cslam:kitti_06}
	\end{subfigure}
	\caption{
		Visualization of BA problems in collaborative SLAM scenarios \cite{Burri16Euroc,Geiger12KITTI}, generated using \software{ORB-SLAM3} \cite{ORBSLAM3}.
		Map points are shown in black.
		Poses of each simulated agent are shown in a distinct color.
	}
	\label{fig:cslam}
\end{figure}

\begin{figure}[h]
	\centering
	\begin{subfigure}[t]{0.45\linewidth}
		\includegraphics[width=\textwidth, trim=250 100 250 100, clip]{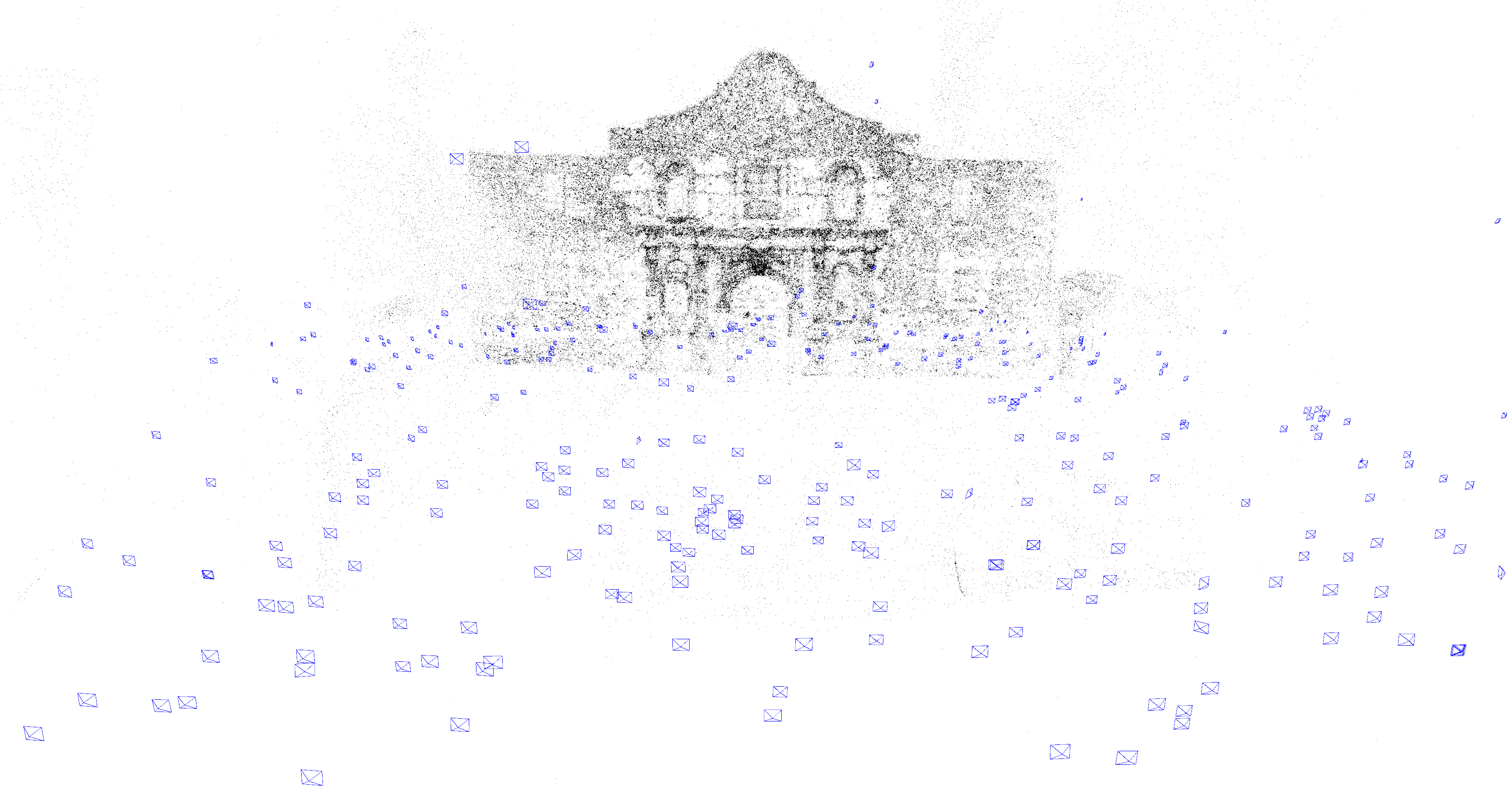}
		\caption{\dataset{Alamo}}
		\label{fig:1dsfm:alamo}
	\end{subfigure}
	~
	\begin{subfigure}[t]{0.45\linewidth}
		\includegraphics[width=\textwidth]{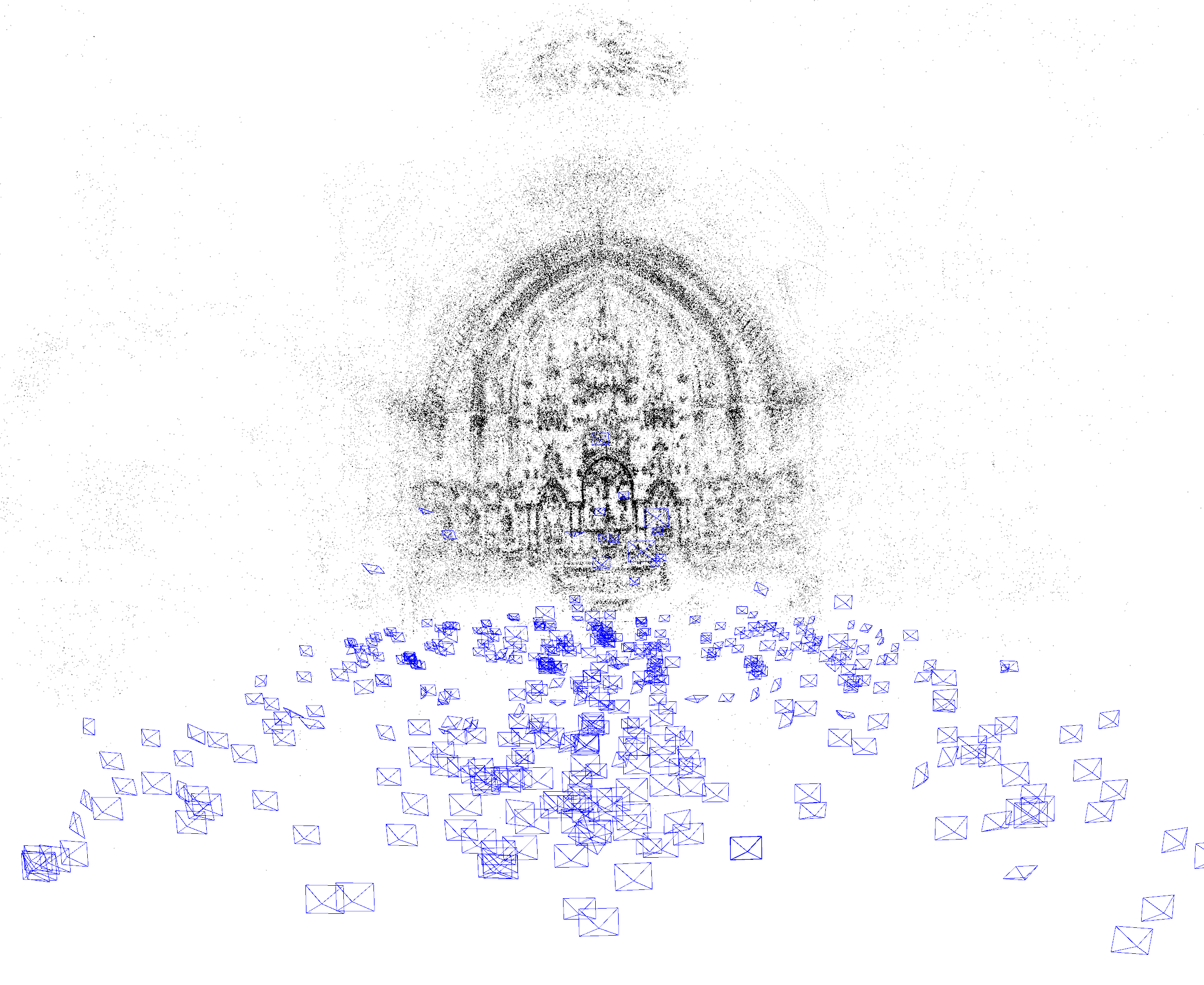}
		\caption{\dataset{Montreal Notre Dame}}
		\label{fig:1dsfm:montreal_notre_dame}
	\end{subfigure}	
	\\
	\begin{subfigure}[t]{0.45\linewidth}
		\includegraphics[width=\textwidth]{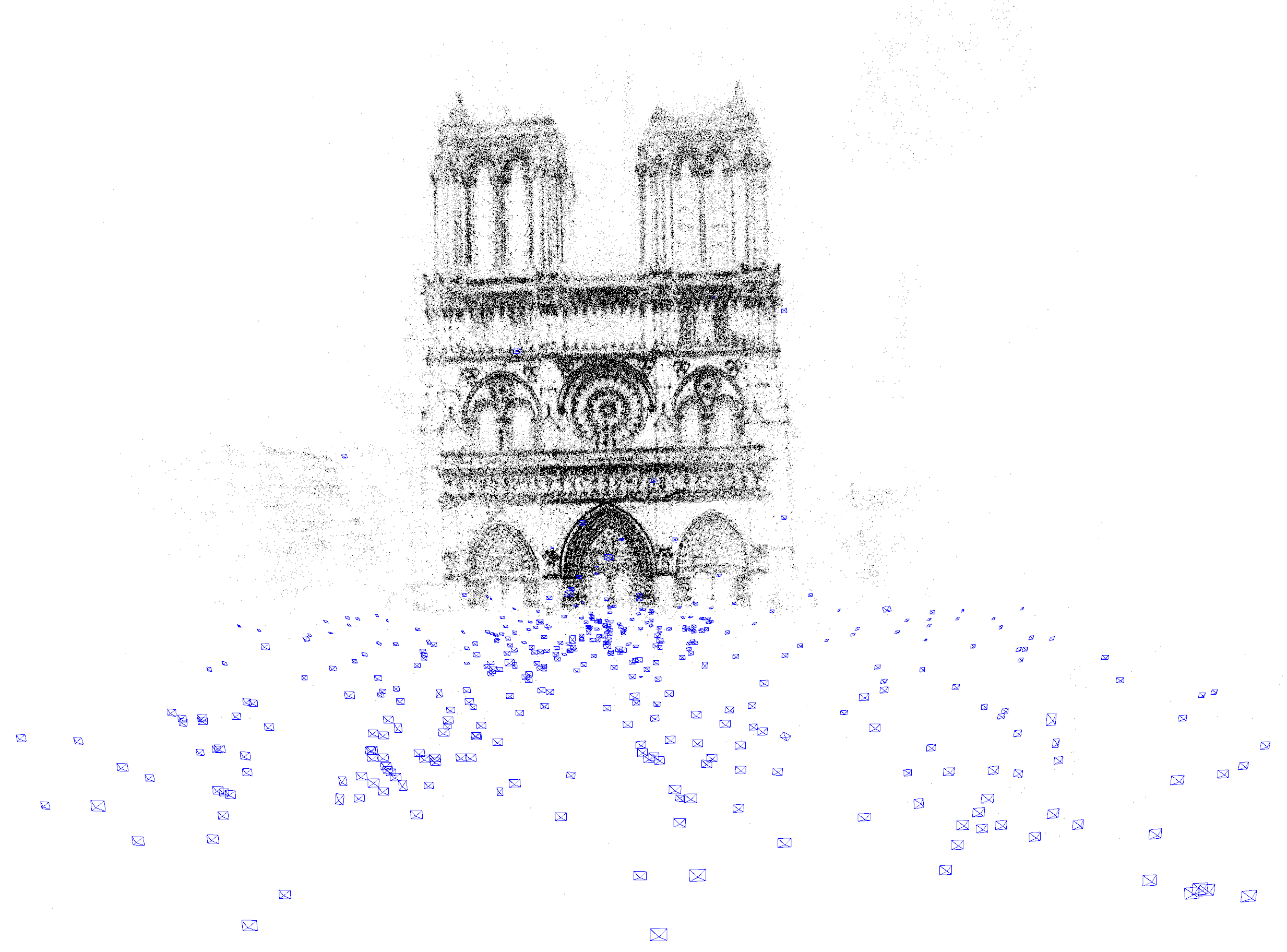}
		\caption{\dataset{Notre Dame}}
		\label{fig:1dsfm:notre_dame}
	\end{subfigure}
	~
	\begin{subfigure}[t]{0.45\linewidth}
		\includegraphics[width=\textwidth]{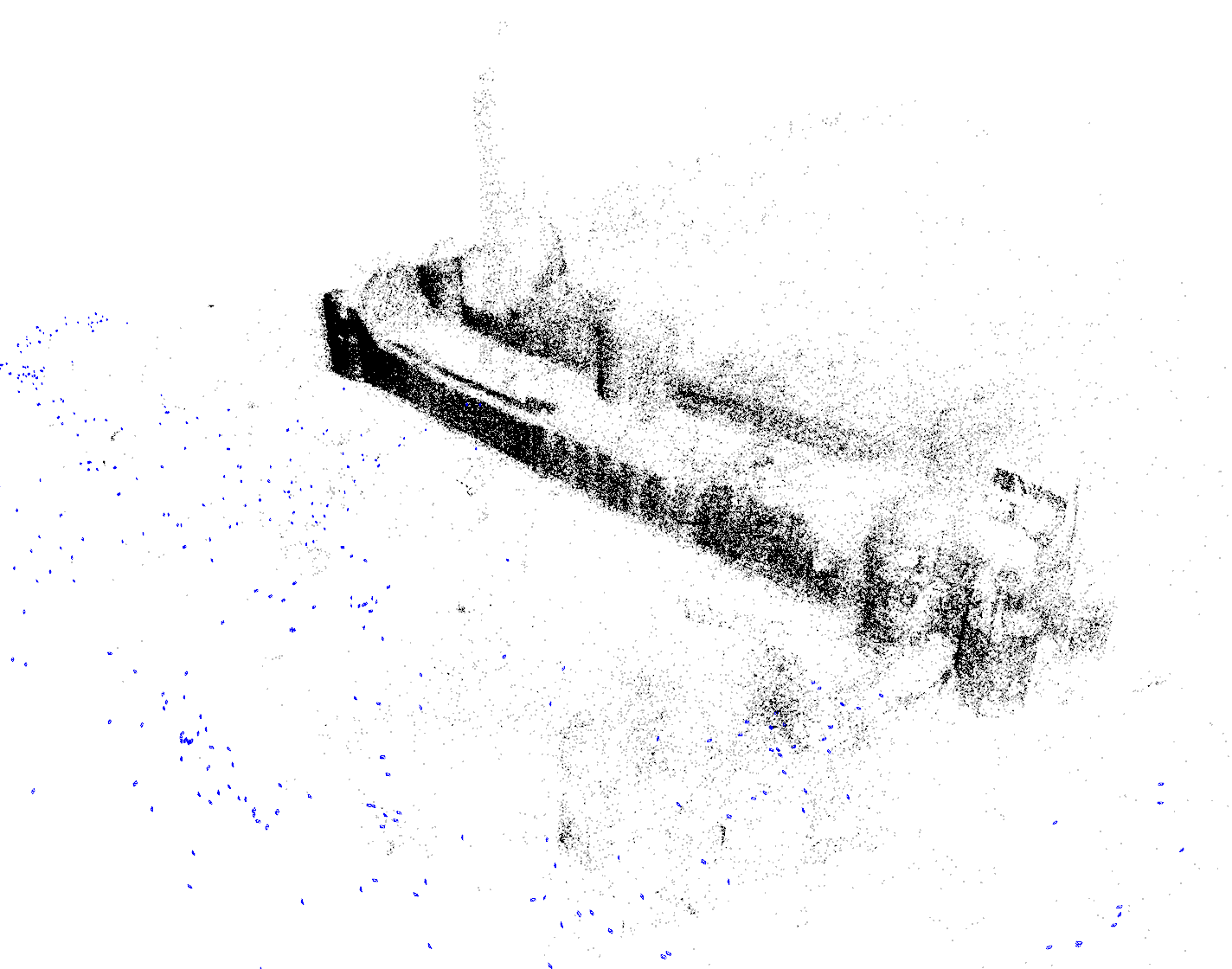}
		\caption{\dataset{Tower of London}}
		\label{fig:1dsfm:tower_of_london}
	\end{subfigure}
	\\
	\begin{subfigure}[t]{0.45\linewidth}
		\includegraphics[width=\textwidth]{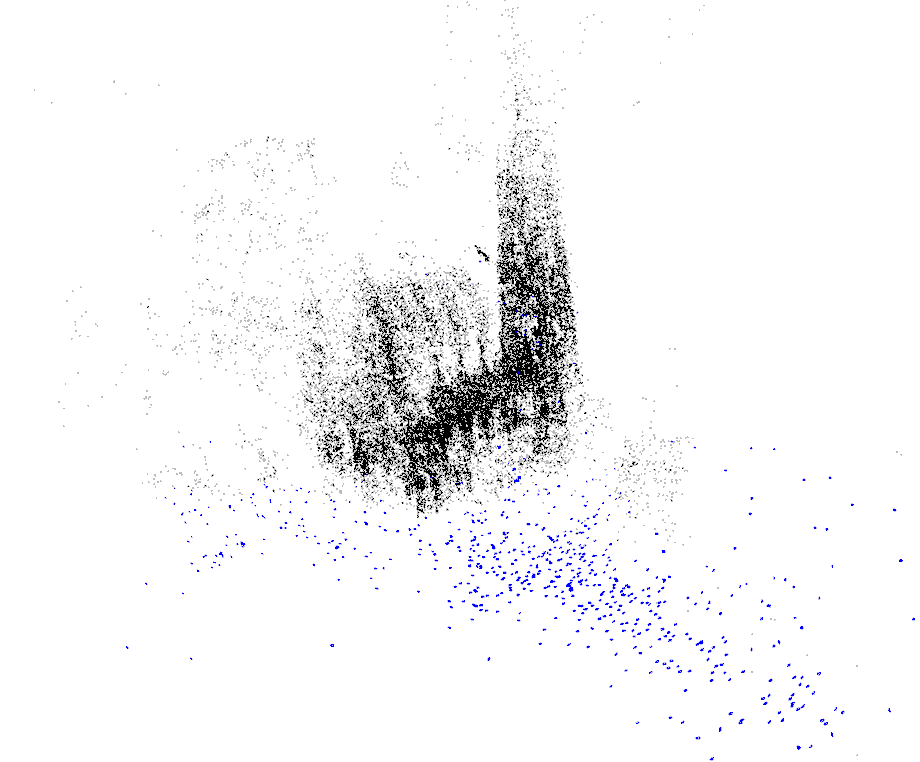}
		\caption{\dataset{Vienna Cathedral}}
		\label{fig:1dsfm:vienna_cathedral}
	\end{subfigure}
	~
	\begin{subfigure}[t]{0.45\linewidth}
		\includegraphics[width=\textwidth,trim=100 100 100 0,clip]{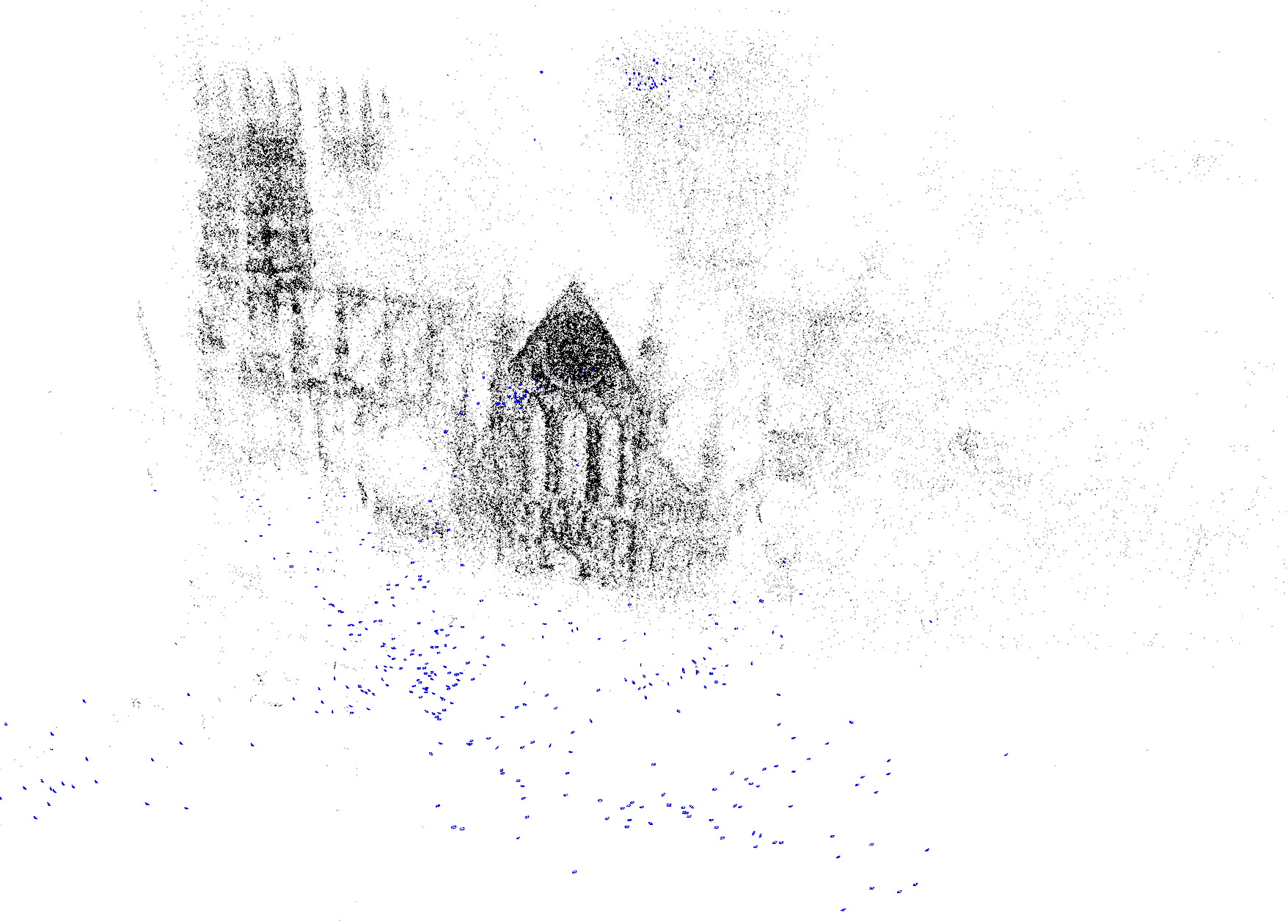}
		\caption{\dataset{Yorkminster}}
		\label{fig:1dsfm:yorkminster}
	\end{subfigure}
	\caption{
		Visualization of BA problems in collaborative SfM scenarios \cite{Wilson141dsfm}, 
		generated using \software{Theia}~\cite{TheiaSfM}. 
		Map points and camera poses are shown in black and blue, respectively.
		Each dataset is partitioned into 50 agents to simulate a collaborative BA problem.
	}
	\label{fig:1dsfm}
\end{figure}

\end{appendices}

\end{document}